\def\isarxiv{1}

\ifdefined\isarxiv
\documentclass[11pt]{article}
\else
\documentclass{article}
\usepackage[final]{neurips_2021}
\fi

\usepackage{tabu}
\usepackage{array}
\usepackage{amsmath}
\usepackage{amsthm}
\usepackage{amssymb}
\usepackage{algorithm}
\usepackage{subfig}
\usepackage{dirtytalk}
\usepackage{color}
\usepackage[english]{babel}
\usepackage{graphicx}
\usepackage{grffile}
\usepackage{wrapfig,epsfig}
\usepackage{epstopdf}
\usepackage{url}
\usepackage{color}
\usepackage{epstopdf}
\usepackage{algpseudocode}
\usepackage[T1]{fontenc}
\usepackage{bbm}
\usepackage{comment}
\usepackage{dsfont}
\usepackage{mathtools}
\usepackage{enumitem}

\usepackage{tikz}
\usepackage{hyperref}  
\hypersetup{
    colorlinks=true,
    citecolor=red,
    linkcolor=blue,
    filecolor=cyan,
    urlcolor=magenta,
}
\usetikzlibrary{arrows}

\ifdefined\isarxiv
\usepackage[margin=1in]{geometry}

\else
\fi

\graphicspath{{./figs/}}

\usepackage[nameinlink,capitalize]{cleveref}
\usepackage{textcomp}
\usepackage{multirow}
\newtheorem{theorem}{Theorem}[section]
\newtheorem{lemma}[theorem]{Lemma}
\newtheorem{definition}[theorem]{Definition}

\newtheorem{proposition}[theorem]{Proposition}
\newtheorem{corollary}[theorem]{Corollary}

\newtheorem{problem}[theorem]{Problem}

\newcommand{\wh}{\widehat}

\newcommand{\ov}{\overline}

\newcommand{\R}{\mathbb{R}}

\renewcommand{\hat}{\wh}

\DeclareMathOperator*{\E}{{\mathbb{E}}}

\DeclareMathOperator{\poly}{poly}

\DeclareMathOperator{\tr}{tr}

\newcommand{\minip}{\mathsf{Min}\text{-}\mathsf{IP}} 
\newcommand{\maxip}{\mathsf{MaxIP}}

\newcommand{\lsh}{$\mathsf{LSH}$}

\newcommand{\ann}{\mathsf{ANN}}

\definecolor{mygreen}{RGB}{80,180,0}
\definecolor{b2}{RGB}{51,153,255}

\allowdisplaybreaks

\title{Breaking the Linear Iteration Cost Barrier for Some Well-known Conditional Gradient Methods Using 
MaxIP Data-structures\footnote{A preliminary version of this paper appeared in Thirty-Fifth Conference on Neural Information Processing Systems (NeurIPS 2021)}}

\ifdefined\isarxiv

\author{
Anshumali Shrivastava\thanks{\texttt{anshumali@rice.edu}. Rice University and ThirdAI Corp.}
\and
Zhao Song\thanks{\texttt{zsong@adobe.com}. Adobe Research.}
\and
Zhaozhuo Xu\thanks{\texttt{zx22@rice.edu}. Rice University.}
}
\else
\author{%
  Zhaozhuo Xu \\
   Rice University\\
  \texttt{zx22@rice.edu} \\
   \And
   Zhao Song \\
   Adobe Research \\
   \texttt{zsong@adobe.com} \\
   \AND
   Anshumali Shrivastava \\
   Rice University and ThirdAI Corp.\\
   \texttt{anshumali@rice.edu} \\
}
\fi

\date{}

\begin{document}

\ifdefined\isarxiv
\begin{titlepage}
  \maketitle
  \begin{abstract}
  Conditional gradient methods (CGM) are widely used in modern machine learning. CGM's overall running time usually consists of two parts: the number of iterations and the cost of each iteration. Most efforts focus on reducing the number of iterations as a means to reduce the overall running time. In this work, we focus on improving the per iteration cost of CGM. The bottleneck step in most CGM is maximum inner product search ($\maxip$), which requires a linear scan over the parameters.  In practice, approximate $\maxip$ data-structures are found to be helpful heuristics. However, theoretically, nothing is known about the combination of approximate $\maxip$ data-structures and CGM. In this work, we answer this question positively by providing a formal framework to combine the locality sensitive hashing type approximate $\maxip$ data-structures with CGM algorithms.  As a result, we show the first algorithm, where the cost per iteration is sublinear in the number of parameters, for many fundamental optimization algorithms, e.g., Frank-Wolfe, Herding algorithm, and policy gradient.

  \end{abstract}
  \thispagestyle{empty}
\end{titlepage}

\newpage
\else
\maketitle
  \begin{abstract}
  Conditional gradient methods (CGM) are widely used in modern machine learning. CGM's overall running time usually consists of two parts: the number of iterations and the cost of each iteration. Most efforts focus on reducing the number of iterations as a means to reduce the overall running time. In this work, we focus on improving the per iteration cost of CGM. The bottleneck step in most CGM is maximum inner product search ($\maxip$), which requires a linear scan over the parameters.  In practice, approximate $\maxip$ data-structures are found to be helpful heuristics. However, theoretically, nothing is known about the combination of approximate $\maxip$ data-structures and CGM. In this work, we answer this question positively by providing a formal framework to combine the locality sensitive hashing type approximate $\maxip$ data-structures with CGM algorithms.  As a result, we show the first algorithm, where the cost per iteration is sublinear in the number of parameters, for many fundamental optimization algorithms, e.g., Frank-Wolfe, Herding algorithm, and policy gradient.

  \end{abstract}
\fi

\section{Introduction}
Conditional gradient methods (CGM), such as Frank-Wolfe and its variants, are well-known optimization approaches that have been extensively used in modern machine learning. For example, CGM has been applied to kernel methods~\cite{blo12,tlr21}, structural learning~\cite{ljsp13} and online learning~\cite{fgm17,lhy21,hk12}. 

\paragraph{Running Time Acceleration in Optimization:}
Recent years have witnessed the success of large-scale machine learning models trained on vast amounts of data. In this learning paradigm, the computational overhead of most successful models is dominated by the optimization process~\cite{dclt18,bmr+20}. Therefore, reducing the running time of the optimization algorithm is of practical importance. The total running time in optimization can be decomposed into two components: (1) the number of iterations towards convergence, (2) the cost spent in each iteration. Reducing the number of iterations requires a better understanding of the geometric proprieties of the problem at hand and the invention of better potential functions to analyze the progress of the algorithm~\cite{j13,gh15,rsps16,rhsps16,ahhl17,lzc+19,wpd+20}. Reducing the cost spent per iteration usually boils down to designing problem-specific discrete data-structures.  In the last few years, we have seen
a remarkable growth of using data-structures to reduce iteration cost~\cite{cls19,lsz19,blss20,dly21,b20,bln+20,jklps20,jswz21,bll+21,sy21,bpsw21,syz21_neurips,hjs+21,szz21}.

\paragraph{MaxIP Data-structures for Iteration Cost Reduction:}
A well-known strategy in optimization, with CGM, is to perform a greedy search over the weight vectors~\cite{j13,gh15,ahhl17,lsz19,fw56} or training samples~\cite{ldh+17,ldl+18} in each iteration. In this situation, the cost spent in each iteration is linear in the number of parameters. In practical machine learning, recent works~\cite{cxs19,cmfgts20,clp+21,dmzs21,xcl+21} formulate this linear cost in iterative algorithms as an approximate maximum inner product search problem ($\maxip$). They speed up the amortized cost per iteration via efficient data-structures from recent advances in approximate $\maxip$~\cite{sl14,sl15_uai,sl15_www,ns15,gkcs16,yhld17,mb18,ztxl19,tzxl19,gslgsck20}. In approximate $\maxip$ data-structures, locality sensitive hashing ({\lsh}) achieves promising performance with efficient random projection based preprocessing strategies~\cite{sl14,sl15_uai,sl15_www,ns15}. Such techniques are widely used in practice for cost reduction in optimization. \cite{cxs19}  proposes an {\lsh} based gradient sampling approach that reduces the total empirical running time of the adaptive gradient descent. \cite{cmfgts20} formulates the forward propagation of deep neural network as a $\maxip$ problem and uses {\lsh} to select a subset of neurons for backpropagation. Therefore, the total running time of neural network training could be reduced to sublinear in the number of neurons. \cite{dmzs21} extends this idea with system-level design for further acceleration, and \cite{clp+21} modifies the {\lsh} with learning and achieves promising acceleration in attention-based language models. \cite{xcl+21} formulates the greedy step in iterative machine teaching (IMT) as a $\maxip$ problem and scale IMT to large datasets with {\lsh}. \cite{syz21_neurips} shows how to use the exact $\maxip$ data-structure to speedup the training of one-hidden layer over-parameterized ReLU neural network.

\paragraph{Challenges of Sublinear Iteration Cost CGM:}
Despite the practical success of cost-efficient iterative algorithms with approximate $\maxip$ data-structure, the theoretical analysis of its combination with CGM is not well-understood. In this paper, we focus on this combination and target answering the following questions: (1) how to transform the iteration step of CGM algorithms into an approximate $\maxip$ problem? (2) how does the approximate error in $\maxip$ affect CGM in the total number of iterations towards convergence? (3) how to adapt approximate $\maxip$ data structure for iterative CGM algorithms?

\paragraph{Our Contributions:}
We propose a theoretical formulation for combining approximate $\maxip$ and convergence guarantees of CGM. In particular, we start with the popular Frank-Wolfe algorithm over the convex hull where the direction search in each iteration is a $\maxip$ problem. Next, we propose a sublinear iteration cost Frank-Wolfe algorithm using {\lsh} type $\maxip$ data-structures. We then analyze the trade-off of approximate $\maxip$ and its effect on the number of iterations needed by CGM to converge. We show that the approximation error caused by {\lsh} only leads to a constant multiplicative factor increase in the number of iterations. As a result, we retain the sub-linearly of {\lsh}, with respect to the number of parameters, and at the same time retain the same asymptotic convergence as CGMs. 

We summarize our complete contributions as follows.

\begin{itemize}[nosep]
    \item We give the first theoretical CGM formulation that achieves provable sublinear time cost per iteration. We also extend this result into Frank-Wolfe algorithm, Herding algorithm, and policy gradient method.
    \item We propose a pair of efficient transformations that formulates the direction search in Frank-Wolfe algorithm as a projected approximate $\maxip$ problem.
    \item We present the theoretical results that the proposed sublinear Frank-Wolfe algorithm asymptotically preserves the same order in the number of iterations towards convergence. Furthermore, we analyze the trade-offs between the saving in iteration cost and the increase in the number of iterations to accelerate total running time.
    \item We identify the problems of {\lsh} type approximate $\maxip$ for cost reduction in the popular CGM methods and propose corresponding solutions.
\end{itemize}

The following sections are organized as below: Section~\ref{sec:related} introduces the related works on data-structures and optimization, Section~\ref{sec:alg_intro} introduces our algorithm associated with the main statements convergence, Section~\ref{sec:prof_overview} provides the proof sketch of the main statements, Section~\ref{sec:conclude} concludes the paper.

\section{Related work}\label{sec:related}
\subsection{Maximum Inner Product Search for Machine Learning}
Maximum Inner Product Search ($\maxip$) is a fundamental problem with applications in machine learning. Given a query $x\in \R^d$ and an $n$-vector dataset $Y \subset \R^d$, $\maxip$ targets at searching for $y\in Y$ that maximizes the inner product $x^\top y$.  The naive $\maxip$ solution takes $O(dn)$ by comparing $x$ with each $y\in Y$. To accelerate this procedure, various algorithms are proposed to reduce the running time of $\maxip$~\cite{sl14,sl15_uai,ns15,sl15_www,gkcs16,yhld17,yldcc18,dyh19,mb18,zlw+18,ztxl19,tzxl19,gslgsck20}. We could categorize the $\maxip$ approaches into two categories: reduction methods and non-reduction methods. The reduction methods use transformations that transform approximate $\maxip$ into approximate nearest neighbor search ($\ann$) and solve it with $\ann$ data-structures. One of the popular data-structure is locality sensitive hashing~\cite{im98,diim04}. 
\begin{definition}[Locality Sensitive Hashing]\label{def:lsh}
Let $\ov{c}>1$ denote a parameter. Let $p_1, p_2\in (0,1)$ denote two parameters and $p_1 > p_2 $. We say a function family $\mathcal{H}$ is $(r,\ov{c} \cdot r,p_1,p_2)$-sensitive if and only if, for any vectors $x,y \in \R^d$, for any $h$ chosen uniformly at random from $\mathcal{H}$, we have:
\begin{itemize}
    \item if $\| x-y\|_2 \leq r$, then $\Pr_{h\sim {\cal H}} [ h(x)=h(y) ] \geq p_1$,
    \item if $ \|x-y\|_2 \geq \ov{c} \cdot r$, then $\Pr_{h\sim {\cal H}} [ h(x)=h(y) ] \leq p_2$.
\end{itemize}
\end{definition}

Here we define the {\lsh} functions for euclidean distance. {\lsh} functions could be used for search in cosine~\cite{c02,ll19} or Jaccard similarity~\cite{llz19,ll21}. \cite{sl14} first shows that $\maxip$ could be solved by $\ell_2$ {\lsh} and asymmetric transformations. After that, \cite{sl15_uai,ns15,sl15_www,yldcc18} propose a series of methods to solve $\maxip$ via {\lsh} functions for other distance measures. Besides {\lsh}, graph-based $\ann$ approaches~\cite{yhld17} could also be used after reduction.

On the other hand, the non-reduction method directly builds data-structures for approximate $\maxip$. \cite{gkcs16,gslgsck20} use quantization to approximate the inner product distance and build codebooks for efficient approximate $\maxip$. \cite{yhld17,dyh19} propose a greedy algorithm for approximate $\maxip$ under computation budgets. \cite{mb18,ztxl19,tzxl19} directly construct navigable graphs that achieve state-of-the-art empirical performance.

Recently, there is a remarkable growth in applying data-structures for machine learning~\cite{biw19,znv+20,bdi+20}. Following the paradigm, approximate $\maxip$ data-structures have been applied to overcome the efficiency bottleneck of various machine learning algorithms. \cite{yhld17}  formulates the inference of a neural network with a wide output layer as a $\maxip$ problem and uses a graph-based approach to reduce the inference time. In the same inference task, \cite{lxj+20} proposes a learnable {\lsh} data-structure that further improves the inference efficiency with less energy consumption. In neural network training, \cite{cmfgts20,clp+21,dmzs21} uses approximate $\maxip$ to retrieve interested neurons for backpropagation. In this way, the computation overhead of gradient updates in neural networks could be reduced. In linear regression and classification models, \cite{cxs19} uses approximate $\maxip$ data-structures to retrieve the samples with large gradient norm and perform standard gradient descent, which improves the total running time for stochastic gradient descent. \cite{xcl+21} proposes a scalable machine teaching algorithm that enables iterative teaching in large-scale datasets. In bandit problem, \cite{yrkpds21} proposes an LSH based algorithm that solves linear bandits problem with sublinear time complexity.

Despite the promising empirical results, there is little theoretical analysis on approximate $\maxip$ for machine learning. We summarize the major reasons as: (1) Besides {\lsh}, the other approximate $\maxip$ data-structures do not provide theoretical guarantees on time and space complexity. (2) Current approaches treat data-structures and learning dynamics separately. There is no joint analysis on the effect of approximate $\maxip$ for machine learning.

\subsection{Projection-free Optimization}

 Frank-Wolfe algorithm~\cite{fw56} is a projection-free optimization method with wide applications in convex~\cite{j13,gh15} and non-convex optimizations~\cite{rsps16,rhsps16}. The procedure of Frank-Wolfe algorithm could be summarized as two steps: (1) given the gradient, find a vector in the feasible domain that has the maximum inner product, (2) update the current weight with the retrieved vector. Formally, given a function $g:\R^d\rightarrow \R$ over a convex set $S$, starting from an initial weight $w^0$, the Frank-Wolfe algorithm updates the weight with learning rate $\eta$ follows:
 \begin{align*}
    s^{t} \leftarrow & ~ \arg\min_{s \in S } \langle s, \nabla g(w^t) \rangle\\
    w^{t+1} \leftarrow & ~ (1-\eta_t) \cdot w^t + \eta_t \cdot s^t.
\end{align*}

 Previous literature focuses on reducing the number of iterations for Frank-Wolfe algorithm over specific domains such as ${\ell}_p$ balls~\cite{j13,gh15,ahhl17,lzc+19}.  There exists less work discussing the reduction of iteration cost in the iterative procedure of Frank-Wolfe algorithm.  In this work, we focus on the Frank-Wolfe algorithm over the convex hull of a finite feasible set. This formulation is more general and it includes recent Frank-Wolfe applications in probabilistic modeling~\cite{blo12,tlr21}, structural learning~\cite{ljsp13} and policy optimization~\cite{lhy21}. 

\section{Our Sublinear Iteration Cost Algorithm}\label{sec:alg_intro}

In this section, we formally present our results on the sublinear iteration cost CGM algorithms. We start with the preliminary definitions of the objective function. Then, we present the guarantees on the number of iteration and cost per iterations for our sublinear CGM algorithms to converge.
\subsection{Preliminaries}
We provide the notations and settings for this paper. 
We start with basic notations for this paper. For a positive integer $n$, we use $[n]$ to denote the set $\{1,2,\cdots, n\}$. For a vector $x$, we use $\| x \|_2 := ( \sum_{i=1}^n x_i^2 )^{1/2}$ to denote its $\ell_2$ norm.

We say a function is convex if
\begin{align*}
L(x)\geq L(y)+\langle \nabla L(y),x-y \rangle.
\end{align*}

We say a function is $\beta$-smooth if
\begin{align*}
L(y)\leq L(x)+\langle \nabla L(x),y-x \rangle+\frac{\beta}{2}\| y-x \|^2_2 .
\end{align*}

Given a set $A=\{x_i\}_{i\in [n]} \subset \R^d$, we say its convex hull $\mathcal{B}(A)$ is the collection of all finite linear combinations $y$ that satisfies $y=\sum_{i\in [n]} a_i\cdot x_i$,
where $a_i\in [0,1]$ for all $i\in[n]$ and $\sum_{i\in [n]}a_i= 1$.  Let $D_{\max}$ denote the maximum diameter of ${\cal B}(A)$ so that $\|x-y\|_2\leq D_{\max}$ for all $(x,y)\in {\cal B}(A)$. We present the detailed definitions in Appendix~\ref{sec:preli}.

Next, we present the settings of our work. Let $S\subset \R^D$ denote a $n$-point finite set. Given a convex and $\beta$-smooth function $g : \R^d \rightarrow \R$ defined over the convex hull ${\cal} B(S)$, our goal is to find a $w\in {\cal} B(S)$ that minimizes $g(w)$. Given large $n$ in the higher dimension, the dominant complexity of iteration cost lies in finding the $\maxip$ of $\nabla g(w)$ with respect to $S$. In this setting, the fast learning rate of Frank-Wolfe in $\ell_p$ balls~\cite{j13,ahhl17,lzc+19} can not be achieved. We present the detailed problem setting of the Frank-Wolfe algorithm in Appendix~\ref{sec:alg}.

\subsection{Our Results}

We present our main results with comparison to the original algorithm in Table~\ref{tab:slow_compare}. From the table, we show that with near-linear preprocessing time, our algorithms maintain the same number of iterations towards convergence while reducing the cost spent in each iteration to be sublinear in the number of possible parameters.

\begin{table}[h]
    \centering
    \begin{tabular}{|l|l|l|l|l|l|l|} \hline
        &  {\bf Statement} & {\bf Preprocess} & {\bf \#Iters} & {\bf Cost per iter}  \\ \hline \hline
        Frank-Wolfe  & \cite{j13} & 0 & $O(\beta D_{\max}^2/\epsilon)$ & $O(dn+{\cal T}_{g})$  \\ \hline
        Ours&  Theorem~\ref{thm:frank_wolfe_lsh_informal}  & $ dn^{1+o(1)} $ & $O(\beta D_{\max}^2/\epsilon)$ & $O(  dn^{\rho} +{\cal T}_{g})$ \\ \hline \hline
        Herding & \cite{blo12} & 0 & $O( D_{\max}^2/\epsilon)$ & $O(dn)$  \\ \hline
        Ours  &Theorem~\ref{thm:herding_lsh_informal} & $dn^{1+o(1)}$ & $O(D_{\max}^2/\epsilon)$ & $O(  dn^{\rho} )$ \\ \hline \hline
        Policy gradient & \cite{lhy21} & 0 & $O(\frac{\beta D_{\max}^2}{\epsilon^2(1-\gamma)^3\mu_{\min}^2})$ & $O(dn+{\cal T}_Q)$  \\ \hline
        Ours & Theorem~\ref{thm:policy_gradient_lsh_informal} & $dn^{1+o(1)}$ & $O(\frac{ \beta D_{\max}^2}{\epsilon^2(1-\gamma)^3\mu_{\min}^2})$ & $O(  dn^{\rho} +{\cal T}_Q)$ \\ \hline
    \end{tabular}
    \caption{Comparison between classical algorithm and our sublinear time algorithm. We compare our algorithm with Frank-Wolfe in: (1) ``Frank-Wolfe'' denotes Frank-Wolfe algorithm~\cite{j13} for convex functions over a convex hull. Let ${\cal T}_g$ denote the time for evaluating the gradient for any parameter.  (2) ``Herding'' denotes kernel Herding algorithm~\cite{blo12} (3) ``Policy gradient'' denotes the projection free policy gradient method~\cite{lhy21}. Let ${\cal T}_Q$ denote the time for evaluating the policy gradient for any parameter. Let $\gamma\in(0,1)$ denote the discount factor. Let $\mu_{\min}$ denote the minimum probability density of a state. Note that $n$ is the number of possible parameters.   $n^{o(1)}$ is smaller than $n^{c}$ for any constant $c>0$. Let $\rho \in (0,1)$ denote a fixed parameter determined by {\lsh} data-structure. The  failure probability of our algorithm is $1/\poly(n)$. $\beta$ is the smoothness factor. $D_{\max}$ denotes the maximum diameter of the convex hull.}
    \label{tab:main_compare}
\end{table}

Next, we introduce the statement of our sublinear iteration cost algorithms. We start by introducing our result for improving the running time of Frank-Wolfe.
\begin{theorem}[Sublinear time Frank-Wolfe, informal of Theorem~\ref{thm:frank_wolfe_lsh_formal}]\label{thm:frank_wolfe_lsh_informal}
Let $g : \R^d \rightarrow \R$ denote a convex and $\beta$-smooth function. Let the complexity of calculating $\nabla g(x)$ to be ${\cal T}_{g}$. Let  $S \subset \R^d$ denote a set of $n$ points. Let ${\cal B} \subset \R^d$ denote the convex hull of $S$ with maximum diameter $D_{\max}$. Let $\rho \in (0,1)$ denote a fixed parameter.   
For any parameters $\epsilon, \delta$, there is an iterative algorithm (Algorithm~\ref{alg:frank_wolfe_lsh_formal}) that takes $  O(dn^{1+o(1)})$ time in pre-processing, takes $T = O({\beta D_{\max}^2}/{\epsilon})$ iterations and $O(dn^{\rho}+{\cal T}_{g})$ cost per iteration, starts from a random $w^0$ from ${\cal B}$ as its initialization point,
and outputs $w^T \in \R^d$ from ${\cal B}$ such that
\begin{align*}
   g(w^T) - \min_{w\in \cal{B}} g(w) \leq \epsilon, 
\end{align*}
holds with probability at least $1-1/\poly(n)$. \end{theorem}

Next, we show our main result for the Herding algorithm. Herding algorithm is widely applied in kernel methods~\cite{cws12}. \cite{blo12} shows that the Herding algorithm is equivalent to a conditional gradient method with the least-squares loss function. Therefore, we extend our results and obtain the following statement.

\begin{theorem}[Sublinear time Herding algorithm, informal version of Theorem~\ref{thm:herding_lsh_formal}]\label{thm:herding_lsh_informal}
Let ${\cal X}\subset \R^d$ denote a feature set and $\Phi: \R^d \rightarrow \R^k$ denote a mapping.   Let $D_{\max}$ denote the maximum diameter of $\Phi({\cal X})$ and ${\cal B}$ be the convex hull of $\Phi({\cal X})$. Given a distribution $p(x)$ over ${\cal X}$, we denote $\mu= \E_{x\sim p(x)}[\Phi(x)]$. Let $\rho\in(0,1)$ denote a fixed parameter.   For any parameters $\epsilon, \delta$, there is an iterative algorithm (Algorithm~\ref{alg:herding_formal})  that takes $  O(dn^{1+o(1)})$ time in pre-processing, takes $T = O({D_{\max}^2}/{\epsilon})$ iterations and $O( dn^{\rho} )$ cost per iteration, starts from a random $w^0$ from ${\cal B}$ as its initialization point,
and outputs $w^T \in \R^k$ from ${\cal B}$ such that
\begin{align*}
   \frac{1}{2} \|w^T-\mu\|_2^2 - \min_{w\in \cal{B}} \frac{1}{2} \|w-\mu\|_2^2 \leq \epsilon, 
\end{align*}
holds with probability at least $1-1/\poly(n)$.
\end{theorem}

Finally, we present our result for policy gradient.  Policy gradient~\cite{slhd+14} is a popular algorithm with wide applications in robotics~\cite{ps06} and recommendation~\cite{cdc+19}. \cite{lhy21} proposes a provable Frank-Wolfe method that maximizes the reward functions with policy gradient. However, the optimization requires a linear scan over all possible actions, which is unscalable in complex environments. We propose an efficient Frank-Wolfe algorithm with per iteration cost sublinear in the number of actions. Our statement is presented below.
\begin{theorem}[Sublinear time policy gradient, informal version of Theorem~\ref{thm:policy_gradient_lsh_formal}]\label{thm:policy_gradient_lsh_informal}
Let ${\cal T}_{Q}$ denote the time for computing the policy graident. Let $D_{\max}$ denote the maximum diameter of action space and $\beta$ is a constant. Let $\gamma\in (0,1)$. Let $\rho\in(0,1)$ denote a fixed parameter.  Let $\mu_{\min}$ denote the minimal density of states in ${\cal S}$. 
There is an iterative algorithm (Algorithm~\ref{alg:policy_gradient_lsh_formal}) that spends $  O(dn^{1+o(1)})$ time in preprocessing, takes $O(\frac{ \beta D_{\max}^2}{\epsilon^2(1-\gamma)^3\mu_{\min}^2})$ iterations and $O( dn^{\rho} +{\cal T}_{Q})$ cost per iterations, starts from a random point $\pi^0_{\theta}$ as its initial point,
and outputs $\pi^T_{\theta}$ that has the average gap $\sqrt{\sum_{s\in {\cal S}}g_{T}(s)^2}< \epsilon$ holds 
with probability at least $1-1/\poly(n)$, where $g_{T}(s)$ is defined in Eq.~\eqref{eq:rl_action_maxip}..
\end{theorem}

\section{Proof Overview}\label{sec:prof_overview}
We present the overview of proofs in this section. We start with introducing the efficient $\maxip$ data-structures. Next, we show how to transform the direction search in a conditional gradient approach into a $\maxip$ problem. Finally, we provide proof sketches for each main statement in Section~\ref{sec:alg_intro}. The detailed proofs are presented in the supplement material.

\subsection{Approximate \texorpdfstring{$\maxip$}{~} Data-structures}\label{sec:lsh_intro}
We present the {\lsh} data-structures for approximate $\maxip$ in this section. The detailed description is presented in Appendix~\ref{sec:preli}.
We use the reduction-based approximate $\maxip$ method with {\lsh} data-structure to achieve sublinear iteration cost. Note that we choose this method due to its clear theoretical guarantee on the retrieval results. It is well-known that an {\lsh} data-structures is used for approximate nearest neighbor problem. The following definition of approximate nearest neighbor search is very standard in literature~\cite{am93,im98,diim04,ainr14,ailrs15,ar15,iw18,alrw17,air18,dirw19,ccd+20}.

\begin{definition}[Approximate Nearest  Neighbor ($\ann$)]\label{def:ann:informal}
Let $\ov{c} >1$ and $r \in (0,2)$ denote two parameters.  Given an $n$-vector set $Y \subset \mathbb{S}^{d-1}$ on a unit sphere, the objective of the $(\ov{c},r)$-Approximate Nearest Neighbor ($\ann$) is to construct a data structure that, for any query $x \in \mathbb{S}^{d-1}$ such that $\min_{y\in Y}\| y - x \|_2 \leq r$, it returns a vector $z$ from $Y$ that satisfies $\| z - x \|_2 \leq \ov{c} \cdot r$.
\end{definition}

In the iterative-type optimization algorithm, the cost per iteration could be dominated by the Approximate $\maxip$ problem (Definition~\ref{def:amaxip}), which is the dual problem of the $(\ov{c},r)$-$\ann$.

\begin{definition}[Approximate $\maxip$]\label{def:amaxip}
Let $c \in (0,1)$ and $\tau \in (0,1)$ denote two parameters.
Given an $n$-vector dataset $Y \subset \mathbb{S}^{d-1}$ on a unit sphere, the objective of the $(c,\tau)$-{$\maxip$} is to construct a data structure that, given a query $x \in \mathbb{S}^{d-1}$ such that $\max_{y\in Y}\langle x , y \rangle \geq \tau$, it retrieves a vector $z$ from $Y$ that satisfies $\langle x , z \rangle \geq c \cdot \max_{y \in Y} \langle x,y \rangle$.
\end{definition}

Next, we present the the primal-dual connection between $\ann$ and approximate $\maxip$. Given to unit vectors $x,y\in \R^d$ with both norm equal to $1$, $\|x-y\|_2^2= 2 - 2\langle x , y\rangle$. Therefore, we could  maximizing  $\langle x , y\rangle$ by minimizing $\|x-y\|_2^2$.  Based on this connection, we present how to solve $(c,\tau)$-{$\maxip$} using $(\ov{c},r)$-{$\ann$}.  We start with showing how to solve $(\ov{c},r)$-{$\ann$} with {\lsh}.

\begin{theorem}[Andoni, Laarhoven, Razenshteyn and Waingarten~\cite{alrw17}]\label{thm:ar17}
Let $\ov{c} > 1$ and $r \in (0,2)$ denote two parameters. One can solve $(\ov{c},r)$-$\ann$ on a unit sphere in query time $O(d \cdot n^{\rho})$ using preprocessing time $O(dn^{1+o(1)})$ and space $O(n^{1+o(1)} + d n)$, where $\rho = \frac{2}{\ov{c}^2} -\frac{1}{\ov{c}^4}+o(1)$.
\end{theorem}

Next, we solve $(c,\tau)$-{$\maxip$} by solving  $(\ov{c},r)$-{$\ann$} using Theorem~\ref{thm:ar17}. We have 

\begin{corollary}[An informal statement of Corollary~\ref{coro:maxip_lsh_formal}]\label{coro:maxip_lsh_informal}
Let $c \in (0,1)$ and $\tau \in(0,1)$ denote two parameters. One can solve $(c,\tau)$-{$\maxip$} on a unit sphere ${\cal S}^{d-1}$ in query time $O(d \cdot n^{\rho})$, where $\rho\in(0,1)$, using {\lsh} with both  preprocessing time and space in $O(dn^{1+o(1)})$.
\end{corollary}

In our work, we consider a generalized form of approximate $\maxip$, denoted as projected approximate $\maxip$.
\begin{definition}[Projected approximate $\maxip$]\label{def:proj_approximate_maxip_informal}
Let $\phi, \psi: \R^d \rightarrow \R^k$ denote two transforms. Given an $n$-vector dataset $Y \subset \R^d $ so that $\psi(Y) \subset \mathbb{S}^{k-1}$, the goal of the $(c,\phi, \psi,\tau)$-{$\maxip$} is to construct a data structure that, given a query $x\in \R^d$ and $\phi(x) \in \mathbb{S}^{k-1}$ such that $\max_{y\in Y}\langle \phi(x) , \psi(y) \rangle \geq \tau$, it retrieves a vector $z \in Y$ that satisfies $\langle \phi(x) , \psi(z) \rangle \geq c \cdot (\phi, \psi)\text{-}\maxip (x,Y)$.
\end{definition}

For details of space-time trade-offs, please refer to Appendix~\ref{sec:data_structure}. The following sections show how to use projected approximate $\maxip$ to accelerate the optimization algorithm by reducing the cost per iteration.

\subsection{Efficient Transformations}\label{sec:transform_intro}

We have learned from Section~\ref{sec:lsh_intro} that $(c,\tau)$-{$\maxip$} on a unit sphere ${\cal S}^{d-1}$ using  {\lsh} for {$\ann$}. Therefore, the next step is to transform the direction search procedure in iterative optimization algorithm into a $\maxip$ on a unit sphere. To achieve this, we formulate the direction search as a projected approximate $\maxip$ (see Definition~\ref{def:projected_maxip}). We start  with presenting a pair of transformation $\phi_0,\psi_0:\R^{d} \rightarrow \R^{d+1}$ such that, given a function $g : \R^d \rightarrow \R$, for any $x,y$ in a convex set $\mathcal{K}$, we have
 
\begin{align}\label{eq:asym_trans_direct}
\phi_0 (x) := &  [\nabla g(x) ^\top, x^\top\nabla g(x)]^\top, ~~~
\psi_0(y) := [ -y^\top,1]^\top.
\end{align}

In this way, we show that
\begin{align}\label{eq:asym_trans_direct_res}
    \langle y-x,\nabla g(x) \rangle =&~-\langle  \phi_0(x) , \psi_0(y) \rangle,\notag\\
    \arg\min_{y\in Y} \langle y-x,\nabla g(x) \rangle=&~\arg\max_{y\in Y} \langle  \phi_0(x) , \psi_0(y) \rangle
\end{align}

Therefore, we could transform the direction search problem into a $\maxip$ problem.

Next, we present a standard transformations~\cite{ns15} that connects the $\maxip$ to $\ann$ in unit sphere. For any $x,y\in \R^d$, we propose transformation $\phi_1,\psi_1:\R^{d} \rightarrow \R^{d+2}$ such that
\begin{align}\label{eq:asym_trans_mips}
  \phi_1(x) =&~ \begin{bmatrix} (D_{x}^{-1}x)^\top & 0 & \sqrt{1-\|x D_{x}^{-1}\|_2^2}  \end{bmatrix}^\top\notag\\
  \psi_1(y) =&~ \begin{bmatrix} (D_{y}^{-1}y)^\top & \sqrt{1-\|yD_{y}^{-1}\|_2^2} & 0 \end{bmatrix}^\top
\end{align}

Here $D_x$, $D_y$ are some constant that make sure both $x/D_x$ and $y/D_y$ have norms less than $1$. Under these transformations, both $\phi_1(x)$ and $\psi_1(y)$ have norm $1$ and $\arg\max_{y\in Y} \langle \phi_1(x),\psi_1(y)\rangle=\arg\max_{y\in Y} \langle x,y\rangle$.

Combining transformations in Eq.~\eqref{eq:asym_trans_direct} and Eq.~\eqref{eq:asym_trans_mips}, we obtain query transform $\phi:\R^d\rightarrow \R^{d+3}$ with form $\phi(x)=\phi_1(\phi_0(x))$ and data transform $\phi:\R^d\rightarrow \R^{d+3}$ with form $\psi(y)=\psi_1(\psi_0(y))$. Using $\phi$ and $\psi$, we transform the  direction search problem in optimization into a $\maxip$ in unit sphere. Moreover, given a set $Y\subset \R^d$ and a query $x\in \R^d$, the solution $z$ of $(c,\phi,\psi,\tau)$-$\maxip$ over $(x,Y)$ has the propriety that $\langle z-x,\nabla g(x)\rangle\leq c\cdot \min_{y\in Y} \langle y-x,\nabla g(x) \rangle$. Thus, we could approximate the direction search with {\lsh} based {$\maxip$} data-structure.

Note that only $\maxip$ problem with positive inner product values could be solved by {\lsh}. We found the direction search problem naturally satisfies this condition. We show that if $g$ is convex, given a set $S \subset \R^d$, we have $\min_{s \in S } \langle \nabla g(x) ,  s-x \rangle\leq 0$ for any $x\in {\cal B}(S)$, where ${\cal B}$ is the convex hull of $S$. Thus, $\max_{y\in Y} \langle  \phi_0(x) , \psi_0(y) \rangle$ is non-negative following Eq.~\eqref{eq:asym_trans_direct_res}.

\subsection{Proof of Theorem~\ref{thm:frank_wolfe_lsh_informal}}

We present the proof sketch for Theorem~\ref{thm:frank_wolfe_lsh_informal} in this section.  We refer the readers to Appendix~\ref{sec:converge} for the detailed proofs.

Let $g : \R^d \rightarrow \R$ denote a convex and $\beta$-smooth function. Let the complexity of calculating $\nabla g(x)$ to be ${\cal T}_{g}$. Let  $S \subset \R^d$ denote a set of $n$ points, and ${\cal B} \subset \R^d$ be the convex hull of $S$ with maximum diameter $D_{\max}$. Let $\phi,\psi:\R^d\rightarrow \R^{d+3}$ denote the tranformations defined in Section~\ref{sec:transform_intro}. Starting from a random vector $w^0\in {\cal B}(S)$. Our sublinear Frank-Wolfe algorithm follows the update following rule that each step
\begin{align*}
    s^{t} \leftarrow & ~  (c,\phi,\psi,\tau)\text{-}\maxip\ \text{of } w^t\ \text{with respect to}\ S \\
    w^{t+1} \leftarrow & ~ w^t +  \eta \cdot (s^t-w^t)
\end{align*}

We start with the upper bounding $\langle s^t-w^t,\nabla g(w^t) \rangle$. Because $s^t$ is the $(c,\phi,\psi,\tau)$-$\maxip$ of $w^t$ with respect to $S$, we have
\begin{align}\label{eq:c_approx_intro}
    \langle s^t-w^t,\nabla g(w^t) \rangle \leq c \min_{s\in S} \langle s-w^t,\nabla g(w^t)
    \leq c \langle w^*-w^t,\nabla g(w^t) \rangle
\end{align}

For convenient of the proof, for each $t$, we define $ h_t = g(w^t) - g(w^*)$. Next, we upper bound $h_{t+1}$ as
\begin{align}\label{eq:bound_h_intro}
    h_{t+1}
    \leq&~ g(w^t)+\eta_t\langle s^t-w^t,\nabla g(w^t) \rangle+\frac{\beta}{2} \eta_t^2\|s^t-w^t\|_2^2 -g(w^*)\notag\\
    \leq&~ g(w^t)+c\eta_t\langle w^*-w^t,\nabla g(w^t) \rangle+\frac{\beta}{2} \eta_t^2\|s^t-w^t\|_2^2 -g(w^*) \notag\\
    \leq&~ g(w^t)+c\eta_t\langle w^*-w^t,\nabla g(w^t) \rangle+\frac{\beta D_{\max}^2}{2}\eta_t^2 -g(w^*) \notag\\
     \leq&~ (1-\eta_t)g(w^t)+c\eta_t g(w^*)+\frac{\beta D_{\max}^2}{2}\eta_t^2 -g(w^*) \notag\\
    =&~ (1-c\eta_t)h_{t}+\frac{\beta D_{\max}^2}{2}\eta_t^2 \notag\\
\end{align}
where the first step follows from the definition of $\beta$-smoothness, the second step follows from Eq.~\eqref{eq:c_approx_intro}, the third step follows from the definition of $D_{\max}$, the forth step follows from the convexity of $g$.

Let $\eta=\frac{2}{c(t+2)}$ and $A_t=\frac{t(t+1)}{2}$. Combining them with Eq.\eqref{eq:bound_h_intro}, we show that 
\begin{align*}
    A_{t+1}h_{t+1}-A_{t}h_{t}=&~c^{-2}\frac{t+1}{t+2}\beta D_{\max}^2\\
    <&~ c^{-2}\beta D_{\max}^2
\end{align*}

Using induction from $1$ to $t$, we show that
\begin{align*}
    A_{t}h_{t}<c^{-2}t\beta D_{\max}^2
\end{align*}

Taken $A_t=\frac{t(t+1)}{2}$ into consideration, we have
\begin{align*}
    h_t<\frac{2\beta D_{\max}^2}{c^{2}(t+1)}
\end{align*}

Given constant approximation ratio $c$, $t$ should be in $O(\frac{\beta D_{\max}^2
}{\epsilon})$ so that $h_t\leq \epsilon$. Thus, we complete the proof.

 {\bf Cost Per Iteration}
After we take $O(dn^{1+o(1)})$ preprocessing time, the cost per iteration consists three pairs: (1) it takes ${\cal T}_{g}$ to compute $\nabla g(w^t)$, (2) it takes $O(d)$ to perform transform $\phi$ and $\psi$, (3) it takes $O(dn^{\rho})$ to retrieve $s^t$ from {\lsh}. Thus, the final cost per iteration would be $O(dn^{\rho}+{\cal T}_{g})$. 

Next, we show how to extend the proof to Herding problem.  Following~\cite{blo12}, we start with defining function $g=\frac{1}{2}\|w^T-\mu\|_2^2$. We show that this function $g$ is a convex  and $1$-smooth function.  Therefore, the Herding algorithm is equivalent to the Frank-Wolfe Algorithm over function $g$. Using the proof of  Theorem~\ref{thm:frank_wolfe_lsh_informal} with $\beta=1$, we show that it takes $T = O({D_{\max}^2}/{\epsilon})$ iterations and $O( dn^{\rho} )$ cost per iteration to reach the $\epsilon$-optimal solution. Similar to Theorem~\ref{thm:frank_wolfe_lsh_informal}, we show that the cost per iteration would be $O(dn^{\rho})$ as it takes $O(d)$ to compute $\nabla g(w^t)$.

\subsection{Proof of Theorem~\ref{thm:policy_gradient_lsh_informal}}
We present the proof sketch for Theorem~\ref{thm:policy_gradient_lsh_informal} in this section.   We refer the readers to Appendix~\ref{sec:policy} for the detailed proofs.

In this paper, we focus on the action-constrained  Markov Decision Process (ACMDP). In this setting, we are provided with a state ${\cal S}\in \R^{k}$ and action space ${\cal A}\in \R^{d}$. However, at each step $t\in \mathbb{N}$, we could only access a finite $n$-vector set of actions ${\cal C}(s) \subset {\cal A}$. Let us assume the ${\cal C}(s)$ remains the same in each step. Let us denote $D_{\max}$ as the maximum diameter of ${\cal A}$.

When you play with this ACMDP, the policy you choose is defined as $\pi_{\theta}(s):{\cal S} \rightarrow {\cal A}$ with parameter $\theta$. Meanwhile, there exists a reward function $r: {\cal S}\times {\cal A}  \in [0,1]$. Then, we define the Q function as below,
\begin{align*}
    Q(s,a|\pi_{\theta})&=\E\Big[\sum_{t=0}^{\infty} \gamma^t r(s_t,a_t)| s_0=s,a_0=a,\pi_{\theta}\Big].
\end{align*}
where $\gamma\in (0,1)$ is a discount factor.

Given a state distribution $\mu$, the objective of policy gradient is to maximize $J(\mu,\pi_{\theta})=\E_{s\sim \mu,a\sim \pi_{\theta}}[Q(s,a|\pi_{\theta})]$ via policy gradient \cite{slhd+14} denoted as:
\begin{align*}
    \nabla_{\theta} J(\mu,\pi_{\theta})= \E_{s\sim d_{\mu}^{\pi}}\Big[\nabla_{\theta}\pi_{\theta}(s) \nabla_{a}Q(s,\pi_{\theta}(s)|\pi_{\theta})|\Big].
\end{align*}

\cite{lhy21} propose an iterative algorithm that perform $\maxip$ at each iteration $k$ over actions to find
\begin{align}\label{eq:rl_action_maxip}
g_{k}(s)=\max_{a \in {\cal C}(s)} \langle a^k_s-\pi^k_{\theta}(s), \nabla_{a} Q(s,\pi^k_{\theta}(s)|\pi^k_{\theta})) \rangle.
\end{align}

In this work, we approximate Eq.~\eqref{eq:rl_action_maxip} using $(c,\phi,\psi,\tau)$-$\maxip$. Here define $\phi: {\cal S}\times \R^d \rightarrow \R^{d+1}$ and $\psi: \R^d \rightarrow \R^{d+1}$ as follows:
\begin{align*}
\phi (s,\pi^k_{\theta}) := & ~ [ \nabla_{a} Q(s,\pi^k_{\theta}(s)|\pi^k_{\theta})^\top, (\pi^k_{\theta})^\top Q(s,\pi^k_{\theta}(s)|\pi^k_{\theta})]^\top ,
\psi(a) := [ a^\top,-1]^\top.
\end{align*}
Then, we have $g_{k}(s) =\langle  \phi(s,\pi^k_{\theta}) , \psi(a) \rangle$. Note that we still require transformations in Eq.~\eqref{eq:asym_trans_mips} to generate unit vectors. 

Next, we show that if we retrieve an action  $\hat{a^k_s}$ using $(c,\phi,\psi,\tau)$-$\maxip$, the gap $\hat{g_{k}}(s)$ would be lower bounded by
\begin{align}\label{eq:g_hat_g_k_intro}
    \hat{g_{k}}(s)=&~\langle \hat{a^k_s}-\pi^k_{\theta}(s), \nabla_{a} Q(s,\pi^k_{\theta}(s)|\pi^k_{\theta})) \rangle\notag\\
    \geq &~ cg_k(s)
\end{align}

Combining Eq.~\eqref{eq:g_hat_g_k_intro} the standard induction in \cite{lhy21}, we upper bound $\sum_{s\in {\cal S}}g_{T}(s)^2$ as
\begin{align}
    \sum_{s\in {\cal S}}g_{T}(s)^2\leq &~ \frac{1}{T+1}\frac{2\beta D_{\max}^2}{c^2(1-\gamma)^3\mu_{\min}^2}.
\end{align}
where $\mu_{\min}$ denotes the minimal density of sates in ${\cal S}$ and $\beta$ is the smoothness factor.

In this way, given a constant factor $c$, if we would like to minimize the gap  $\sum_{s\in {\cal S}}g_{T}(s)^2< \epsilon^2$, $T$ should be $O(\frac{\beta D_{\max}^2}{\epsilon^2(1-\gamma)^3\mu_{\min}^2})$.

{\bf Cost Per Iteration}
After we take $O(dn^{1+o(1)})$ preprocessing time, the cost per iteration consists three pairs: (1) it takes ${\cal T}_{Q}$ to compute policy gradient, (2) it takes $O(d)$ to perform transform $\phi$ and $\psi$, (3) it takes $O(dn^{\rho})$ to retrieve actions from {\lsh}. Thus, the final cost per iteration would be $O(dn^{\rho}+{\cal T}_{Q})$.

\subsection{Quantization for Adaptive Queries}
In optimization, the gradient computed in every iteration is not independent of each other. This would generate a problem for $\maxip$ data-structures. If we use a vector containing the gradients as a query for $\maxip$ data-structures, the query failure probability in each iteration is not independent. Therefore, the total failure probability could not be union bounded. As previous $\maxip$ data-structures focus on the assumptions that queries are independent, the original failure analysis could not be directly applied.

This work uses a standard query quantization method to handle the adaptive query sequence in optimization. Given the known query space, we quantize it by lattices~\cite{gs88}. This quantization is close to the Voronoi diagrams. In this way, each query is located into a cell with a center vector. Next, we perform a query using the center vector in the cell. Therefore, the failure probability of the $\maxip$ query sequence is equivalent to the probability that any center vector in the cell fails to retrieve its approximate $\maxip$ solution.  As the centers of cells are independent, we could union bound this probability. On the other hand, as the maximum diameter of the cell is $\lambda$, this query quantization would introduce a $\lambda$ additive error in the inner product retrieved.  We refer the readers to Appendix~\ref{sec:adaptive} for the detailed quantization approach.

\subsection{Optimizing Accuracy-Efficiency Trade-off}
In this work, we show that by {\lsh} based $\maxip$ data-structure, the cost for direction search is $O(dn^\rho)$, where $\rho\in (0,1)$. In Section~\ref{sec:sublinear_fw} of the supplementary material, we show that $\rho$ is a function of constant $c$ and parameter $\tau$ in approximate $\maxip$ (see Definition~\ref{def:amaxip}). Moreover, we also show in Section~\ref{sec:sublinear_fw} that {\lsh} results in only a constant multiplicative factor increase in the number of iterations. Considering the cost per iteration and the number of iterations, we show that when our algorithms stop at the $\epsilon$-optimal solution, {\lsh} could achieve acceleration in the overall running time. Therefore, we could set $c$ and $\tau$ parameters to balance the accuracy-efficiency trade-off of CGM to achieve the desired running time.

\section{Concluding Remarks}\label{sec:conclude}
In this work, we present the first Frank-Wolfe algorithms that achieve sublinear linear time cost per iteration. We also extend this result into Herding algorithm and policy gradient methods. We formulate the direction search in Frank-Wolfe algorithm as a projected approximate maximum inner product search problem with a pair of efficient transformations. Then, we use locality sensitive hashing data-structure to reduce the iteration cost into sublinear over the number of possible parameters. Our theoretical analysis shows that the sublinear iteration cost Frank-Wolfe algorithm preserves the same order in the number of iterations towards convergence. Moreover, we analyze and optimize the trade-offs between saving iteration cost and increasing the number of iterations to achieve sublinear total running time. Furthermore, we identify the problems of existing $\maxip$ data-structures for cost reduction in iterative optimization algorithms and propose the corresponding solutions. We hope this work can be the starting point of future study on sublinear iteration cost algorithms for optimization.

\section*{Acknowledgements}
This work was supported by National Science
Foundation IIS-1652131, BIGDATA-1838177, AFOSR-YIP
FA9550-18-1-0152, ONR DURIP Grant, and the ONR
BRC grant on Randomized Numerical Linear Algebra. The authors would like to thank Beidi Chen for the helpful discussion on optimization. The authors would like to thank Lichen Zhang for valuable discussions about data-structures.

\ifdefined\isarxiv
 \bibliographystyle{alpha}
\else
 \bibliographystyle{unsrt}
\fi
\bibliography{ref}

\appendix
\newpage
\section*{Appendix}

\paragraph{Roadmap.} We provide supplementary materials for our work. Section~\ref{sec:preli} introduces the preliminary notations and definitions, Section~\ref{sec:data_structure} introduces the {\lsh} data structure in detail for $\maxip$, Section~\ref{sec:alg} presents our sublinear Frank-Wolfe algorithm, Section~\ref{sec:converge} presents the convergence analysis for sublinear Frank-Wolfe, Section~\ref{sec:herding} provides the algorithm and analysis on sublinear cost Herding algorithm, Section~\ref{sec:policy} provides the algorithm and analysis on sublinear cost policy gradient approach,  Section~\ref{sec:adaptive} shows how to handle adaptive queries in $\maxip$.

\section{Preliminary}\label{sec:preli}

\subsection{Notations}
We use $\Pr[]$ and $\E[]$ for probability and expectation. We denote $\max\{a,b\}$ as the maximum between $a$ and $b$. We denote $\min \{a,b\}$ (resp. $\max\{a,b\}$) as the minimum (reps. maximum) between $a$ and $b$. For a vector $x$, we denote $\| x \|_2 := ( \sum_{i=1}^n x_i^2 )^{1/2}$ as its $\ell_2$ norm. We denote $\| x \|_p := (\sum_{i=1}^n |x_i|^p)^{1/p}$ as its $\ell_p$ norm. For a square matrix $A$, we denote $\tr[A]$ as the trace of matrix $A$.

\subsection{\texorpdfstring{{\lsh}}{~} and \texorpdfstring{$\maxip$}{~}}

We start with the defining the Approximate  Nearest Neighbor ($\ann$) problem~\cite{am93,im98,diim04,ainr14,ailrs15,ar15,iw18,alrw17,air18,dirw19,ccd+20} as:
\begin{definition}[Approximate Nearest  Neighbor ($\ann$)]\label{def:ann:formal}
Let $\ov{c} >1$ and $r \in (0,2)$ denote two parameters.  Given an $n$-vector set $Y \subset \mathbb{S}^{d-1}$ on a unit sphere, the objective of the $(\ov{c},r)$-Approximate Nearest Neighbor ($\ann$) is to construct a data structure that, for any query $x \in \mathbb{S}^{d-1}$ such that $\min_{y\in Y}\| y - x \|_2 \leq r$, it returns a vector $z$ from $Y$ that satisfies $\| z - x \|_2 \leq \ov{c} \cdot r$.
\end{definition}

The $\ann$ problem can be solved via locality sensitive hashing ({\lsh})~\cite{im98,diim04,iw18}. In this paper, we use the standard definitions of {\lsh} (see Indyk and Motwani~\cite{im98}).
\begin{definition}[Locality Sensitive Hashing]
Let $\ov{c}>1$ denote a parameter. Let $p_1, p_2\in (0,1)$ denote two parameters and $p_1 > p_2 $. We say a function family $\mathcal{H}$ is $(r,\ov{c} \cdot r,p_1,p_2)$-sensitive if and only if, for any vectors $x,y \in \R^d$, for any $h$ chosen uniformly at random from $\mathcal{H}$, we have:
\begin{itemize}
    \item if $\| x-y\|_2 \leq r$, then $\Pr_{h\sim {\cal H}} [ h(x)=h(y) ] \geq p_1$,
    \item if $ \|x-y\|_2 \geq \ov{c} \cdot r$, then $\Pr_{h\sim {\cal H}} [ h(x)=h(y) ] \leq p_2$.
\end{itemize}
\end{definition}

Next, we show that {\lsh} solves $\ann$ problem with sublinear query time complexity.

\begin{theorem}[Andoni, Laarhoven, Razenshteyn and Waingarten \cite{alrw17}]\label{thm:ar17:formal}
Let $\ov{c} > 1$ and $r \in (0,2)$ denote two parameters. One can solve $(\ov{c},r)$-$\ann$ on a unit sphere in query time $O(d \cdot n^{\rho})$ using preprocessing time $O(dn^{1+o(1)})$ and space $O(n^{1+o(1)} + d n)$, where $\rho = \frac{2}{\ov{c}^2} -\frac{1}{\ov{c}^4}+o(1)$.
\end{theorem}
Here we write $o(1)$ is equivalent to $O(1/\sqrt{\log n})$. Note that we could reduce $d$ to $n^{o(1)}$ with Johnson–Lindenstrauss Lemma~\cite{jl84}. Besides, we could achieve better $\rho$ using {\lsh} in~\cite{ar15} if we allowed to have more proprocessing time.

In this work, we focus on a well-known problem in computational complexity: approximate $\maxip$. In this work, we follow the standard notation in \cite{c18} and define the approximate $\maxip$ problem as follows:

\begin{definition}[Approximate $\maxip$]\label{def:approximate_maxip}
Let $c \in (0,1)$ and $\tau \in (0,1)$ denote two parameters.
Given an $n$-vector dataset $Y \subset \mathbb{S}^{d-1}$ on a unit sphere, the objective of the $(c,\tau)$-{$\maxip$} is to construct a data structure that, given a query $x \in \mathbb{S}^{d-1}$ such that $\max_{y\in Y}\langle x , y \rangle \geq \tau$, it retrieves a vector $z$ from $Y$ that satisfies $\langle x , z \rangle \geq c \cdot \max_{y \in Y} \langle x,y \rangle$.
\end{definition}

In many applications, it is more convenient to doing inner product search in a transformed/projected space compared to doing inner product search in the original space. Thus, we propose the following definitions (Definition~\ref{def:projected_maxip} and Definition~\ref{def:proj_approximate_maxip})
\begin{definition}[Projected $\maxip$]\label{def:projected_maxip} 
Let $\phi, \psi: \R^d \rightarrow \R^k$ denote two transforms. Given a data set $Y\subseteq \R^d$ and a point $x\in\R^d$, we define $(\phi, \psi)$-$\maxip$ as follows:
\begin{align*}
    (\phi, \psi)\text{-}\maxip (x,Y) := \max_{y \in Y} \langle \phi(x),\psi(y) \rangle
\end{align*}
\end{definition}

\begin{definition}[Projected approximate $\maxip$]\label{def:proj_approximate_maxip}
Let $\phi, \psi: \R^d \rightarrow \R^k$ denote two transforms. Given an $n$-vector dataset $Y \subset \R^d $ so that $\psi(Y) \subset \mathbb{S}^{k-1}$, the goal of the $(c,\phi, \psi,\tau)$-{$\maxip$} is to construct a data structure that, given a query $x\in \R^d$ and $\phi(x) \in \mathbb{S}^{k-1}$ such that $\max_{y\in Y}\langle \phi(x) , \psi(y) \rangle \geq \tau$, it retrieves a vector $z \in Y$ that satisfies $\langle \phi(x) , \psi(z) \rangle \geq c \cdot (\phi, \psi)\text{-}\maxip (x,Y)$.
\end{definition}

Besides $\maxip$, We also define a version of the minimum inner product search problem.
\begin{definition}[Regularized $\minip$]\label{def:reg_minip} 
Given a data set $Y\subseteq \R^d$ and a point $x\in\R^d$. Let $\phi:\R^d\rightarrow \R^d$ denote a mapping. Given a constant $\alpha$, we define regularized $\minip$ as follows:
\begin{align*}
    (\phi,\alpha)\text{-}\minip (x,Y) := \min_{y \in Y} \langle y-x,\phi(x) \rangle + \alpha\|x-y\|.
\end{align*}
\end{definition}

\subsection{Definitions and  Properties for Optimization}

We start with listing definitions for optimization.
\begin{definition}[Convex hull and its diameter]\label{def:cvx_hull}
Given a set $A=\{x_i\}_{i\in [n]} \subset \R^d$, we define its convex hull $\mathcal{B}(A)$ to be the collection of all finite linear combinations $y$ that satisfies  $ y=\sum_{i\in [n]} a_i\cdot x_i$ where $a_i\in [0,1]$ for all $i\in[n]$ and $\sum_{i\in [n]}a_i= 1$.  Let $D_{\max}$ denote the maximum square of diameter of ${\cal B}(A)$ so that $\|x-y\|_2\leq D_{\max}$ for all $(x,y)\in {\cal B}(A)$.
\end{definition}

\begin{definition}[Smoothness]\label{def:smooth}
We say $L$ is $\beta$-smooth if 
\begin{align*}
L(y)\leq L(x)+\langle \nabla L(x),y-x \rangle+\frac{\beta}{2}\| y-x \|^2_2
\end{align*}
\end{definition}

\begin{definition}[Convex]\label{def:convex}
We say function $L$ is convex if 
\begin{align*}
L(x)\geq L(y)+\langle \nabla L(y),x-y \rangle
\end{align*}
\end{definition}

Next, we list properties for optimization.
\begin{corollary}\label{coro:hull_maxip}
For a set $A=\{x_i\}_{i\in [n]} \subset \R^d$, and its convex hull $\mathcal{B}(A)$, given a query $q\in \R^d$, if $x^*=\arg\max_{x\in A} q^\top x$. Then, $q^\top y \leq q^\top x^*$ for all $y\in \mathcal{B}(A)$.
\end{corollary}
\begin{proof}
We can upper bound $q^\top y$ as follows:
\begin{align*}
    q^\top y 
    =&~q^\top (\sum_{i\in [n]}a_i\cdot x_i)  \\
    = &~ \sum_{i\in [n]}a_i\cdot q^\top x_i\\
    \leq &~ \sum_{i\in [n]}a_i\cdot q^\top x^*\\
    \leq &~ q^\top x^*
\end{align*}
where the first step follows from the definition of convex hull in Definition~\ref{def:cvx_hull}, the second step is an reorganization, the third step follows the fact that $a_i\in [0,1]$ for all $i\in[n]$ and $q^\top x_i \leq q^\top x^*$ for all $x_i\in A$, the last step follows that $\sum_{i\in [n]}a_i= 1$.
\end{proof}

\begin{lemma}[$\maxip$ condition] \label{lemma:max_ip_condition}
Let $g : \R^d \rightarrow \R$ denote a convex function. Let $S \subset \R^d$ denote a set of points. Given a vector $x\in {\cal B}(S)$, we have
\begin{align*}
\min_{s \in S } \langle \nabla g(x) ,  s-x \rangle  \leq 0, ~~~ \forall x\in \cal{B}(S).
\end{align*}
\end{lemma}

\begin{proof}
Let $s_{\min}=\arg\min_{s\in S} \langle \nabla g(x) ,  s \rangle $. Then, we upper bound $\langle \nabla g(x) ,  s_{\min}-x \rangle$ as
\begin{align}\label{eq:maxip_s_0}
    \langle \nabla g(x) ,  s_{\min}-x \rangle =&~\langle \nabla g(x) ,  s_{\min}-\sum_{s\in S} a_i\cdot s \rangle\notag\\
    \leq &~\langle \nabla g(x) ,  \sum_{s\in S}a_i( s_{\min}- s_i) \rangle \notag\\
    = &~\sum_{s_i\in S}a_i \langle \nabla g(x) ,   s_{\min}- s_i \rangle \notag\\
    = &~\sum_{s_i\in S}a_i (\langle \nabla g(x) ,   s_{\min} \rangle -\langle \nabla g(x) ,   s_i \rangle )\notag\\
    \leq&~ 0
\end{align}
where the first step follows from the definition of convex hull in Definition~\ref{def:cvx_hull}, the second and third steps are reorganizations, the final steps follows that $\langle \nabla g(x) ,   s_0 \rangle \leq \langle \nabla g(x) ,   s \rangle$ for all $s\in S$.

Next, we upper bound $\min_{s \in S } \langle \nabla g(x) ,  s-x \rangle  \leq 0, ~~~ \forall x\in \cal{B}(S)$ as
\begin{align*}
    \min_{s \in S } \langle \nabla g(x) ,  s-x \rangle  \leq  \langle \nabla g(x) ,  s_0-x \rangle  \leq 0
\end{align*}
where the first step follows from the definition of function $\min$ and the second step follows from Eq~\eqref{eq:maxip_s_0}.
\end{proof}

\section{Data Structures}\label{sec:data_structure}

In this section, we present a formal statement that solves $(c,\tau)$-$\maxip$ problem on unit sphere using {\lsh} for $(\ov{c},r)$-$\ann$.

\begin{corollary}[Formal statement of Corollary~\ref{coro:maxip_lsh_informal}]\label{coro:maxip_lsh_formal}
Let $c \in (0,1)$ and $\tau \in(0,1)$. Given a set of $n$-vector set $Y \subset {\cal S}^{d-1}$ on the unit sphere, there exists a data structure with $O(dn^{1+o(1)})$ preprocessing time and $O(n^{1+o(1)} + d n)$ space so that for any query $x \in {\cal S}^{d-1}$, we take $O(d\cdot n^{\rho})$ query time to retrieve the $(c,\tau)$-$\maxip$ of $x$ in $Y$ with probability at least $0.9$\footnote{It is obvious to boost probability from constant to $\delta$ by repeating the data structure $\log(1/\delta)$ times.}, where $\rho:=  \frac{2(1-\tau)^2}{(1-c\tau)^2}-\frac{(1-\tau)^4}{(1-c\tau)^4}+o(1)$
\end{corollary}

\begin{proof}
We know that $\|x-y\|_2^2= 2 - 2\langle x , y\rangle$ for all $x,y\in {\cal S}^{d-1}$. In this way,  if we have a {\lsh} data-structure for $(\ov{c}, r)$-$\ann$. It could be used to solve $(c, \tau)$-$\maxip$ with $\tau = 1-0.5 r^2$ and $c = \frac{1-0.5 \ov{c}^2 r^2}{1 - 0.5 r^2}$. Next, we write $\ov{c}^2$ as
\begin{align*}
\ov{c}^2= \frac{ 1 - c(1-0.5 r^2) }{0.5r^2} = \frac{1 - c \tau }{1-\tau} .
\end{align*}

Next, we show that if  the {\lsh} is initialized following Theorem~\ref{thm:ar17:formal}, it takes query time $O(d \cdot n^{\rho})$, space $O(n^{1+o(1)} + d n)$ and preprocessing time $O(dn^{1+o(1)})$ to solve $(c,\tau)$-$\maxip$ through solving $(\ov{c}, r)$-$\ann$,   where
\begin{align*}
    \rho =  \frac{2}{\ov{c}^2} -\frac{1}{\ov{c}^4} +o(1) = \frac{2(1-\tau)^2}{(1-c\tau)^2}-\frac{(1-\tau)^4}{(1-c\tau)^4}+o(1). 
\end{align*}
\end{proof}

In practice, $c$ is increasing as we set parameter $\tau$ close to $\maxip(x,Y)$. 
There is also another {\lsh} data structure~\cite{ar15} with longer preprocessing time and larger space that could solve the $(c, \tau)$-$\maxip$ with similar query time complexity. We refer readers to Section 8.2 in~\cite{ssx21} for more details. Moreover, Corrolary~\ref{coro:maxip_lsh_formal} could be applied to projected $\maxip$ problem.

\begin{corollary}[]\label{coro:proj_maxip_lsh}
Let $c \in (0,1)$ and $\tau \in(0,1)$. Let $\phi, \psi: \R^d \rightarrow \R^k$ denote two transforms.   Let ${\cal T}_{\phi}$ denote the time to compute $\phi(x)$ and ${\cal T}_{\psi}$ denote the time to compute $\psi(y)$. Given a set of $n$-points $Y\in \R^d$ with $\psi(Y) \subset {\cal S}^{k-1}$ on the sphere, one can construct a data structure with $O(dn^{1+o(1)}+{\cal T}_{\psi}n)$ preprocessing time and $O(n^{1+o(1)} + d n)$ space so that for any query $x \in \R^d$ with $\phi(x)\in {\cal S}^{k-1}$, we take query time complexity $O(d\cdot n^{\rho}+{\cal T}_{\phi})$ to solve $(c,\phi,\psi,\tau)$-$\maxip$ with respect to $(x,Y)$ with probability at least $0.9$, where $\rho:=  \frac{2(1-\tau)^2}{(1-c\tau)^2}-\frac{(1-\tau)^4}{(1-c\tau)^4}+o(1)$.
\end{corollary}

\begin{proof}

The preprocessing phase can be decomposed in two parts.
\begin{itemize}
    \item It takes $O({\cal T}_{\psi}n)$ time to transform every $y\in Y$ into $\psi(y)$.
    \item It takes $O(O(dn^{1+o(1)})$ time and $O(dn^{1+o(1)}+dn)$ to index every  $\psi(y)$ into {\lsh} using Corrolary~\ref{coro:maxip_lsh_formal}.
\end{itemize}

The query phase can be decomposed in two parts.
\begin{itemize}
    \item It takes $O({\cal T}_{\phi})$ time to transform every $x\in \R^d$ into $\phi(x)$.
    \item It takes $O(d\cdot n^{\rho})$ time perform query for  $\phi(x)$ in {\lsh} using Corrolary~\ref{coro:maxip_lsh_formal}.
\end{itemize}
\end{proof}

\section{Algorithms}\label{sec:alg}

\subsection{Problem Formulation}
In this section, we show how to use Frank-Wolfe Algorithm to solve the Problem~\ref{problem:frankwolfe}.

\begin{problem}\label{problem:frankwolfe}
\begin{align}
    \min_{w\in {\cal B}}&~ g(w)
\end{align}

We have the following assumptions:
\begin{itemize}
    \item $g:\R^d\rightarrow \R$ is a differentiable function.
    \item $S\subset \R^d$ is a finite feasible set. $|S|=n$.
    \item ${\cal B}={\cal B}(S)\subset \R^{d}$ is the convex hull of the finite set $S\subset \R^d$  defined in Definition~\ref{def:cvx_hull}.
    \item $D_{\max}$ is the maximum diameter of ${\cal B}(S)$ defined in Definition~\ref{def:cvx_hull}.
\end{itemize}
\end{problem}

In Problem~\ref{problem:frankwolfe}, function $g$ could have different proprieties about convexity and smoothness.

To solve this problem, we introduce a Frank-Wolfe Algorithm shown in Algorithm~\ref{alg:frank_wolfe_formal}.

\begin{algorithm}[H]\caption{Frank-Wolf algorithm for Problem~\ref{problem:frankwolfe}}\label{alg:frank_wolfe_formal}
\begin{algorithmic}[1]
\Procedure{FrankWolfe}{$S \subset \R^d$}
\State $T\leftarrow O(\frac{\beta D_{\max}^2}{\epsilon})$, $\forall t \in [T]$
\State $\eta \leftarrow \frac{2}{t+2}$
\State Start with $w^0 \in {\cal B}$. \Comment{${\cal B}={\cal B}(S) $(see Definition~\ref{def:cvx_hull}).}
\For{ $t=1 \to T-1$}
    \State $s^t \leftarrow \arg\min_{s\in S} \langle \nabla g(w^t),s \rangle  $\label{line:argmin}
    \State $w^{t+1} \leftarrow (1-\eta_t)w^{t}+\eta_t s^t$
\EndFor
\State \Return $w^{T}$
\EndProcedure
\end{algorithmic}
\end{algorithm}

One of the major computational bottleneck of Algorithm~\ref{alg:frank_wolfe_formal} is the cost paid in each iteration.  Algorithm~\ref{alg:frank_wolfe_formal} has to linear scan all the $s\in S$ in each iteration. To tackle this issue, we propose a Frank-Wolfe Algorithm with sublinear cost in each iteration. 

\subsection{Sublinear Frank-Wolfe Algorithm}

In this section, we present the Frank-Wolfe algorithm with sublinear cost per iteration using \textit{\lsh}. The first step is to formulate the line~\ref{line:argmin} in Algorithm~\ref{alg:frank_wolfe_formal} as a projected $\maxip$ problem defined in Definition~\ref{def:projected_maxip}. To achieve this, we present a general $\maxip$ transform.
\begin{proposition}[$\maxip$ transform]\label{prop:poly_ip_transform}
Let $\phi_1,\psi_1: \R^d \rightarrow \R^{k_1}$ and $\phi_2,\psi_2: \R^d \rightarrow \R^{k_2}$ to be the projection functions. Given the polynomial function $p(z) =\sum_{i=0}^D a_{i} z^{i}$, we show that
\begin{align}
  \langle \phi_1(x), \psi_1(y) \rangle + p(\| \phi_2(x)- \psi_2(y)\|_2^2)=\langle \phi(x),\psi(y) \rangle
\end{align}
where $\phi,\psi: \R^d \rightarrow \R^{k_1+k_2(D+1)^2}$ is the decomposition function.
\end{proposition}

\begin{proof}
Because $\phi_2(x), \psi_2(y)\in \R^{k_2}$, $\| \phi_2(x)- \psi_2(y)\|_2^{2i}=\sum_{j=1}^{k_2} (\phi_2(x)_j- \psi_2(y)_j)^{2i}$. This is the sum over dimensions. Then, we have
\begin{align*}
    p(\| \phi_2(x)- \psi_2(y)\|_2^2)=&~\sum_{i=0}^D a_{i} \| \phi_2(x)- \psi_2(y)\|_2^{2i} \notag\\
    =&~\sum_{i=0}^D a_{i} \sum_{j=1}^{k_2} (\phi_2(x)_j- \psi_2(y)_j)^{2i}
\end{align*}
where the first follows from definition of polynomial $p$, and the second step follows from definition of $\ell_2$ norm. 

Here $\phi_2(x)_j$ means the $j$th entry of $\phi_2(x)$. Using the binomial theorem, we decompose $(\phi_2(x)_j- \psi_2(y)_j)^{2i}$ as:
\begin{align*}
    & ~(\phi_2(x)_j- \psi_2(y)_j)^{2i} \notag \\
    =
    &~ \sum_{l=0}^{2i} \tbinom{2i}{l} \phi_2(x)_j^{2i-l} \psi_2(y)_j^l\notag\\
    =&~\langle  \underbrace{[\phi_2(x)_j^{2i},\cdots,\phi_2(x)_j^{2i-l},\cdots,\phi_2(x)_j,1 ]}_{u_j} ,
    \underbrace{[1,\psi_2(y)_j,\cdots,\psi_2(y)_j^l,\cdots ,\psi_2(y)_j^{2i}]}_{v_j}
    \rangle 
\end{align*}

Then, we generate two vectors $u^i \in \R^{k_2(2i+1)}$ and $v^i \in \R^{k_2(2i+1)}$
\begin{align*}
    u^i = & \begin{bmatrix} u_1 & \cdots & u_j & \cdots & u_{k_2}  \end{bmatrix} \quad & u_j =& \begin{bmatrix} \phi_2(x)_j^{2i} & \cdots & \phi_2(x)_j^{2i-l} & \cdots & \phi_2(x)_j & 1 \end{bmatrix}^\top\\
    v^i = &  \begin{bmatrix} v_1 & \cdots & v_j & \cdots & v_{k_2} \end{bmatrix} \quad & v_j=& \begin{bmatrix} 1 & \psi_2(y)_j & \cdots & \psi_2(y)_j^l & \cdots & \psi_2(y)_j^{2i} \end{bmatrix}^\top
\end{align*}

Thus, $\sum_{j=1}^{k_2} (\phi_2(x)_j- \psi_2(y)_j)^{2i}$ can be rewrite with inner product by concatenating all the $u_j$ together and then concatenating all the $v_j$.
\begin{align*}
    \sum_{j=1}^{k_2} (\phi_2(x)_j- \psi_2(y)_j)^{2i}
    = \langle u^i, v^i \rangle.
\end{align*}

We make vectors $b \in \R^{k_2(D+1)^2}$ and $c \in \R^{k_2(D+1)^2}$ such as
\begin{align*}
    b=&~[u^0\cdots,u^i,\cdots,u_{D} ] \\
    c=&~[a_0v^0,\cdots,a_iv^i,\cdots,a_Dv^D ]
\end{align*}

So that
\begin{align*}
    \sum_{i=0}^D a_{i} \sum_{j=1}^{k_2} (\phi_2(x)_j- \psi_2(y)_j)^{2i} =\sum_{i=0}^D a_{i} \langle u^i,v^i \rangle=\sum_{i=0}^D  \langle u^i,a_{i}v^i \rangle= \langle b,c \rangle
\end{align*}

Finally, we have
\begin{align*}
     \langle \phi_1(x), \psi_1(y) \rangle +p(\| \phi_2(x)- \psi_2(y)\|_2^2)  
    =&~ \langle \phi_1(x), \psi_1(y) \rangle+\langle b, c \rangle\notag \\
    =&~\langle[\phi_1(x),b],[\psi_1(y),c]\rangle\notag\\
    =&~\langle\phi(x),\psi(y)\rangle
\end{align*}

Total projected dimension:
\begin{align*}
    k_1+\sum_{i=0}^Dk_2(2i+1)=&~k_1+(D+1)k_2+2k_2\sum_{i=1}^{D}i\notag\\
    =&~k_1+(D+1)k_2+2k_2\cdot \frac{D(D+1)}{2}\notag\\
    =&~ k_1+ k_2(D+1)^2
\end{align*}
\end{proof}

Therefore, any binary function with format 
$\langle \phi_1(x), \psi_1(y) \rangle + p(\| \phi_2(x)- \psi_2(y)\|_2^2)$ defined in Proposition~\ref{prop:poly_ip_transform} can be transformed as a inner product. 

Next, we show that a modified version of line~\ref{line:argmin} in Algorithm~\ref{alg:frank_wolfe_formal} can be formulated as a projected $\maxip$ problem.
\begin{corollary}[Equivalence between projected $\maxip$ and $\minip$]\label{coro:reg_ip_transform}
Let $g$ be a differential function defined on convex set $\mathcal{K} \subset \R^d$. Given $\eta\in (0,1)$ and $x,y\in \mathcal{K}$, we define  $\phi,\psi: \R^d \rightarrow \R^{d+3}$ as follows:
\begin{align*}
    \phi(x):= \begin{bmatrix}\frac{\phi_0(x)^\top}{D_x}  & 0 & \sqrt{1-\frac{\|\phi_0(x)\|_2^2}{D_x^2}} \end{bmatrix}^\top~~~\psi(y) := \begin{bmatrix} \frac{\psi_0(y)^\top}{D_y} & \sqrt{1-\frac{\|\psi_0(y)\|_2^2}{D_y^2}}  & 0 \end{bmatrix}^\top
\end{align*}
where 
\begin{align*}
\phi_0 (x) := &  [\nabla g(x) ^\top, x^\top\nabla g(x)]^\top ~~~
\psi_0(y) := [ -y^\top,1]^\top,
\end{align*}
$D_x$ is the maximum diameter of $\phi_0 (x)$ and $D_y$ is the maximum diameter of $\psi_0 (y)$.

Then, for all $x,y\in \R^d$, we transform them into unit vector $\phi(x)$ and $\psi(y)$ on ${\cal S}^{d+2}$. Moreover, we have
\begin{align*}
     \langle y-x,\nabla g(x) \rangle =-D_xD_y\langle  \phi(x) , \psi(y) \rangle 
\end{align*}

Further, the $(\phi, \psi)$-$\maxip$ (Definition~\ref{def:projected_maxip}) is equivalent to the $(\nabla g,0)$-$\minip$ (Definition~\ref{def:reg_minip}).
\begin{align*}
    \arg \max_{y\in\mathcal{K}} \langle  \phi(x) , \psi(y) \rangle =\arg \min_{y\in\mathcal{K}}  \langle y-x,\nabla g(x) \rangle 
\end{align*}

In addition, let ${\cal T}_{\psi}$ denote the time of evaluating at any point $y \in \R^d$ for function $\psi$, then we have ${\cal T}_{\psi} = O(1)$.

Let ${\cal T}_{\phi}$ denote the time of evaluating at any point $x \in \R^d$ for function $\phi$, then we have ${\cal T}_{\phi} = {\cal T}_{\nabla g} + O(d)$, where the ${\cal T}_{\nabla g}$ denote the time of evaluating function $\nabla g$ at any point $x \in \R^d$.
\end{corollary}

\begin{proof}
We start with showing that $\|\phi(x)\|_2=\|\psi(y)\|_2=1$. Next, we show that
\begin{align}
    \langle  \phi(x) , \psi(y) \rangle 
    =&~ \frac{\langle  \phi_0(x) , \psi_0(y) \rangle }{D_xD_y} \notag\\
    =&~ \frac{\langle -y,\nabla g(x) \rangle+\langle x,\nabla g(x) \rangle}{D_xD_y} \notag\\
    =&~ -\frac{\langle y-x,\nabla g(x) \rangle}{D_xD_y} \notag
\end{align}
where the first step follows from definition of $\phi$ and $\psi$, the second step follows from definition of $\phi_0$ and $\psi_0$,  the last step is a reorganization.

Based on the results above,
\begin{align}
    \arg \max_{y\in\mathcal{K}} \langle  \phi(x) , \psi(y) \rangle =\arg \min_{y\in\mathcal{K}}  \langle y-x,\nabla g(x) \rangle \notag
\end{align}
\end{proof}

Using Corollary~\ref{coro:reg_ip_transform}, the direction search in Frank-Wolfe algorithm iteration is equivalent to a $(\phi, \psi)$-$\maxip$ problem. In this way, we propose Algorithm~\ref{alg:frank_wolfe_lsh_formal}, an Frank-Wolfe algorithm with sublinear cost per iteration using {\lsh}.

\begin{algorithm}[H]\caption{Sublinear Frank-Wolfe for Problem~\ref{problem:frankwolfe}}\label{alg:frank_wolfe_lsh_formal}
\begin{algorithmic}[1]
\State {\bf data structure} \textsc{LSH} \Comment{Corollary~\ref{coro:proj_maxip_lsh}}
    \State \hspace{4mm} \textsc{Init}($S \subset \R^d$, $n \in \mathbb{N}$, $d \in \mathbb{N}$, $c\in (0,1)$) \State \Comment{$|S|=n$, $c\in(0,1)$ is {\lsh} parameter, and $d$ is the dimension of data}
    \State \hspace{4mm} \textsc{Query}($x \in \R^d$, $\tau\in(0,1)$) \Comment{$\tau\in(0,1)$ is {\lsh} parameter}
\State {\bf end data structure}
\State
\Procedure{SublinearFrankWolfe}{$S \subset \R^d$, $n \in \mathbb{N}$, $d\in \mathbb{N} $, $c\in (0,1)$, $\tau\in(0,1)$} \Comment{Theorem~\ref{thm:frank_wolfe_lsh_formal}}
\State Construct $\phi, \psi : \R^d \rightarrow \R^{d+1}$ as Corollary~\ref{coro:reg_ip_transform}
\State {\bf static} \textsc{LSH} \textsc{lsh}
\State \textsc{lsh}.\textsc{Init}($\psi(S),n,d+3,c$)
\State Start with $w^0 \in {\cal B}$. \Comment{${\cal B}={\cal B}(S) $(see Definition~\ref{def:cvx_hull}).}
\State $T\leftarrow O(\frac{\beta D_{\max}^2}{c^2\epsilon})$
\State $\eta \leftarrow \frac{2}{c(t+2)}$, $\forall t \in [T]$
 
\For{ $t=1 \to T-1$}
    \State {\color{blue}/* Query with $w^t$ and retrieve its $(c,\phi,\psi,\tau)$-$\maxip$ $s^t\in S$ from {\lsh} data structure */}
    \State $s^t \leftarrow \textsc{lsh}.\textsc{Query}(\phi(w^t),\tau)$ 
    \State {\color{blue}/* Update $w^t$ in the chosen direction*/}
    \State $w^{t+1} \leftarrow (1-\eta_t) \cdot w^{t} + \eta_t \cdot s^t$
\EndFor
\State \Return $w^{T}$
\EndProcedure
\end{algorithmic}
\end{algorithm}

\section{Convergence Analysis}\label{sec:converge}
In this Section~\ref{sec:converge}, analyze the convergence of our Sublinear Frank-Wolfe algorithm in Algorithm~\ref{alg:frank_wolfe_lsh_formal} when $g$  is  convex (see Definition~\ref{def:convex}) and $\beta$-smooth (see Definition~\ref{def:smooth}). Moreover, we compare  our sublinear Frank-Wolfe algorithm  with Frank-Wolfe algorithm in Algorithm~\ref{alg:frank_wolfe_formal} in terms of number of iterations and cost per iteration.

\subsection{Summary}
We first show the comparsion results in Table~\ref{tab:slow_compare}. We list the statement, preprocessing time, number of iterations and cost per iteration for our algorithm and original Frank-Wolfe algorithm to converge.  As shown in the table, with $O(d n^{1+o(1)}\cdot \kappa )$ preprocessing time,  Algorithm~\ref{alg:frank_wolfe_lsh_formal} achieves $O(dn^{\rho}\cdot \kappa +{\cal T}_{g})$ cost per iteration with $\frac{1}{c^2}$ more iterations.
\begin{table}[h]
    \centering
    \begin{tabular}{|l|l|l|l|l|l|} \hline
        {\bf Algorithm}&{\bf Statement} & {\bf Preprocessing} & {\bf \#iters} & {\bf cost per iter}  \\ \hline
        Algorithm~\ref{alg:frank_wolfe_formal} & \cite{j13}& 0 & $O(\beta D_{\max}^2/\epsilon)$ & $O(dn+{\cal T}_{g})$  \\ \hline
        Algorithm~\ref{alg:frank_wolfe_lsh_formal} &Theorem~\ref{thm:frank_wolfe_lsh_formal} & $O( dn^{1+o(1)}\cdot \kappa)$ & $O(c^{-2}\beta D_{\max}^2/\epsilon)$ & $O(dn^{\rho}\cdot \kappa +{\cal T}_{g})$ \\ \hline
    \end{tabular}
    \caption{Comparison between original Frank-Wolfe algorithm and our sublinear Frank-Wolfe algorithm. Here ${\cal T}_{g}$ denotes the time for computing gradient of $g$, $c\in (0,1)$ is the approximation factor of {\lsh}. We let $\kappa :=  \Theta ( \log(T/\delta) )$ where $T$ is the number of iterations and $\delta$ is the failure probability. $\rho\in(0,1)$ is a fixed parameter determined by {\lsh}.}
    \label{tab:slow_compare}
\end{table}

\subsection{Convergence of Sublinear Frank-Wolfe Algorithm}\label{sec:sublinear_fw}

The goal of this section is to prove Theorem~\ref{thm:frank_wolfe_lsh_formal}.
\begin{theorem}[Convergence result of Sublinear Frank-Wolfe, a formal version of Theorem~\ref{thm:frank_wolfe_lsh_informal}]\label{thm:frank_wolfe_lsh_formal}
Let $g : \R^d \rightarrow \R$ denote a convex (see Definition~\ref{def:convex}) and $\beta$-smooth function (see Definition~\ref{def:smooth}). Let the complexity of calculating $\nabla g(x)$ to be ${\cal T}_{g}$.  Let $\phi, \psi: \R^d \rightarrow \R^k$ denote two transforms in Corollary~\ref{coro:reg_ip_transform}.  Let  $S \subset \R^d$ denote a set of points with $|S| = n$, and ${\cal B} \subset \R^d$ is the convex hull of $S$ (see Definition~\ref{def:cvx_hull}). For any parameters $\epsilon, \delta$, there is an iterative algorithm with that takes $O(dn^{1+o(1)}\cdot\kappa)$  preprocessing time and $O((n^{1+o(1)}+dn)\cdot\kappa)$ space, takes $T = O(\frac{\beta D_{\max}^2}{\epsilon})$ iterations and $O (dn^{\rho}\cdot\kappa+{\cal T}_{g})$ cost per iteration, starts from a random $w^0$ from ${\cal B}$ as initialization point, updates the $w$ in each iteration as follows:
\begin{align*}
    s^{t} \leftarrow & ~  (c,\phi,\psi,\tau)\text{-}\maxip\ \text{of } w^t\ \text{with respect to}\ S \\
    w^{t+1} \leftarrow & ~ w^t +  \eta \cdot (s^t-w^t)
\end{align*}
and outputs $w^T \in \R^d$ from ${\cal B}$ such that
\begin{align*}
   g(w^T) - \min_{w\in \cal{B}} g(w) \leq \epsilon, 
\end{align*}
holds with probability at least $1-\delta$. Here $\kappa:= \Theta ( \log(T/\delta) )$ and $\rho:=  \frac{2(1-\tau)^2}{(1-c\tau)^2}-\frac{(1-\tau)^4}{(1-c\tau)^4}+o(1)$.
\end{theorem}

\begin{proof}

{\bf Convergence.}

Let $t$ denote some fixed iteration. We consider two cases: 
\begin{itemize}
    \item {\bf Case 1.} $\tau > \max_{s\in S }\langle \psi(s), \phi(w^t) \rangle$; 
    \item {\bf Case 2.} $\tau \leq \max_{s\in S }\langle \psi(s), \phi(w^t) \rangle$.
\end{itemize}

\paragraph{Case 1.}

In this case, we can show that 
\begin{align*}
   \tau \geq & ~ \max_{s\in S }\langle \psi(s), \phi(w^t)  \rangle \\
   \geq & ~ \frac{\langle \psi(w^*), \phi(w^t)   \rangle}{D_xD_y} \\
   = & ~ \frac{\langle w^t - w^*, \nabla g(w^t) \rangle}{D_xD_y} \\
   \geq & ~ \frac{g(w^t) - g(w^*)}{D_xD_y},
\end{align*}
where the first step follows from Corollary~\ref{coro:reg_ip_transform}, the second step follows from the Corollary~\ref{coro:hull_maxip}, the third step is a reorganization, the last step follows the convexity of $g$ (see Definition~\ref{def:convex}). 

Thus, as long as $\tau \geq D_xD_y\epsilon$, then we have
\begin{align*}
    g(w^t) - g(w^*) \leq \epsilon.
\end{align*}

This means we already converges to the $\epsilon$-optimal solution.

\paragraph{Case 2.}

We start with the upper bounding $\langle s^t-w^t,\nabla g(w^t) \rangle$ as

\begin{align}\label{eq:c_approx}
    \langle s^t-w^t,\nabla g(w^t) \rangle 
    =&~ -D_xD_y\langle \psi(s^t), \phi(w^t)   \rangle\notag\\
    \leq&~ -c\cdot D_xD_y\max_{s\in S }\langle \psi(s), \phi(w^t)   \rangle  \notag\\
    \leq&~ -c \cdot D_xD_y\langle \psi(w^*), \phi(w^t)   \rangle \notag \\
    =&~c \langle w^*-w^t,\nabla g(w^t) \rangle
\end{align}
where the first step follows from Corollary~\ref{coro:reg_ip_transform}, the second step follows from  Corollary~\ref{coro:proj_maxip_lsh} and $\maxip$ condition in Lemma~\ref{lemma:max_ip_condition}, the third step follows from Corollary~\ref{coro:hull_maxip}.

For convenient of the proof, for each $t$, we define $h_t$ as follows:
\begin{align}\label{eq:def_h_t_lsh}
    h_t = g(w^t) - g(w^*).
\end{align}

Next, we upper bound $h_{t+1}$ as
\begin{align}\label{eq:bound_h_t}
    h_{t+1}&=g(w^{t+1})-g(w^*)\notag\\
    &=g((1-\eta_t)w^{t}+\eta_t s^t)-g(w^*)\notag\\
    &\leq g(w^t)+\eta_t\langle s^t-w^t,\nabla g(w^t) \rangle+\frac{\beta}{2} \eta_t^2\|s^t-w^t\|_2^2 -g(w^*)\notag\\
    &\leq g(w^t)+\eta_t\langle s^t-w^t,\nabla g(w^t) \rangle+\frac{\beta D_{\max}^2}{2}\eta_t^2 -g(w^*)\notag\\
    &\leq g(w^t)+c\eta_t\langle w^*-w^t,\nabla g(w^t) \rangle+\frac{\beta D_{\max}^2}{2}\eta_t^2 -g(w^*) \notag\\
     &= (1-\eta_t)g(w^t)+c\eta_t\left(g(w^t)+\langle w^*-w^t,\nabla g(w^t) \rangle\right)+\frac{\beta D_{\max}^2}{2}\eta_t^2 -g(w^*) \notag\\
     &\leq (1-\eta_t)g(w^t)+c\eta_t g(w^*)+\frac{\beta D_{\max}^2}{2}\eta_t^2 -g(w^*) \notag\\
     &\leq (1-c\eta_t)g(w^t)-(1-c\eta_t) g(w^*)+\frac{\beta D_{\max}^2}{2}\eta_t^2 \notag\\
     &\leq (1-c\eta_t)h_{t}+\frac{\beta D_{\max}^2}{2}\eta_t^2 \notag\\
\end{align}
where the first step follows from definition of $h_{t+1}$ (see Eq.~\eqref{eq:def_h_t_lsh}), the second step follows from the update rule of Frank-Wolfe, the third step follows from the definition of $\beta$-smoothness in Definition~\ref{def:smooth}, the forth step follows from the definition of maximum diameter in Definition~\ref{def:cvx_hull}, the fifth step follows the Eq~\eqref{eq:c_approx}, the sixth step is a reorganization, the seventh step follows from the definition of convexity (see Definition~\ref{def:convex}), the eighth step follows from merging the coefficient of $g(w^*)$, and the last step follows from definition of $h_t$ (see Eq.~\eqref{eq:def_h_t_lsh}).

Let $e_t=A_th_t$, $A_t$ is a parameter and we will decide it later. we have:
\begin{align}
    e_{t+1}-e_t&=A_{t+1}\left( (1-c\eta_t)h_{t}+\frac{\beta D_{\max}^2}{2}\eta_t^2\right) -A_{t} h_{t} \notag \\
    &=\left(A_{t+1}(1-c\eta_t)-A_{t}\right)h_{t}+\sigma+ \frac{\beta D_{\max}^2}{2}A_{t+1}\eta_t^2
\end{align}

Let $A_t=\frac{t(t+1)}{2}$, $c\eta_t=\frac{2}{t+2}$. In this way we rewrite $A_{t+1}(1-\eta_t)-A_{t}$ and $A_{t+1}\frac{\eta_t^2}{2}$ as
\begin{itemize}
    \item $A_{t+1}(1-\eta_t)-A_{t}=0$
    \item $A_{t+1}\frac{\eta_t^2}{2}=\frac{t+1}{(t+2)c^2}<c^{-2}$
\end{itemize}

Next, we upper bound $e_{t+1}-e_t$ as:
\begin{align}\label{eq:original_fw_et_final_lsh}
    e_{t+1}-e_t
    <&~0+c^{-2}\frac{t+1}{t+2}\beta D_{\max}^2\notag\\
    <&~c^{-2}\beta D_{\max}^2
\end{align}
where the first step follows from $A_{t+1}(1-\eta_t)-A_{t}=0$ and $A_{t+1}\frac{\eta_t^2}{2}=\frac{t+1}{(t+2)c^2}$. The second step follows from $\frac{t+1}{t+2}<1$

Based on Eq~\eqref{eq:original_fw_et_final_lsh}, we upper bound $e_t$ using induction and have
\begin{align}
    e_t<c^{-2}t\beta D_{\max}^2
\end{align}

Using the definition of $e_t$, we have
\begin{align}
     h_t=\frac{e_t}{A_t}<\frac{2\beta D_{\max}^2}{c^{2}(t+1)}
\end{align}

To make $h_t\leq \epsilon$, $t$ should be in $O(\frac{\beta D_{\max}^2
}{c^2\epsilon})$. Thus, we complete the proof.

  {\bf Preprocessing time}
  According to Corrollary~\ref{coro:proj_maxip_lsh},  can construct $\kappa=\Theta ( \log(T/\delta) )$ {\lsh} data structures for $(c,\phi,\psi,\tau)$-$\maxip$ with $\phi,\psi$ defined in Corollary~\ref{coro:reg_ip_transform}. As transforming every $s\in S$ into $\psi(s)$ takes $O(dn)$. Therefore, the total the preprocessing time complexity is $O( d n^{1+o(1)}\cdot \kappa)$.
 
 {\bf Cost per iteration}
  Given each $w^t$, compute $\nabla g(w^t)$ takes ${\cal T}_{g}$. Next, it takes $O(d)$ time to generate $\phi(w^t)$ according to Corollary~\ref{coro:reg_ip_transform} based on $g(w^t)$ and $\nabla g(w^t)$. Next, according to Corrollary~\ref{coro:proj_maxip_lsh},  it takes $O(dn^\rho \cdot \kappa)$ to retrieve $s^t$ from $\kappa=\Theta ( \log(T/\delta) )$ {\lsh} data structures. After we select $s^t$, it takes $O(d)$ time to update the $w^{t+1}$. Combining the time for gradient calculation, {\lsh} query and $w^t$ update, the total complexity is $O(dn^{\rho}\cdot \kappa  +{\cal T}_{g})$ with $\rho:=  \frac{2(1-\tau)^2}{(1-c\tau)^2}-\frac{(1-\tau)^4}{(1-c\tau)^4}+o(1)$.
 \end{proof}

\newpage
\section{Herding Algorithm}\label{sec:herding}

\subsection{Problem Formulation}
In this section, we focus on the Herding algorithm a specific example of Problem~\ref{problem:frankwolfe}. We consider a finite set ${\cal X}\subset \R^d$ and a mapping $\Phi: \R^d \rightarrow \R^k$. Given a distribution $p(x)$ over ${\cal X}$, we denote $\mu \in \R^k$ as
\begin{align}\label{eq:herding_target}
    \mu=\E_{x\sim p(x)}[\Phi(x)] 
\end{align}
The goal of Herding algorithm~\cite{cws12} is to find $T$ elements $\{x_1, x_2, \cdots, x_T \}$ $ \subseteq {\cal X}$ such that $\|\mu-\sum_{t=1}^{T} v_t \Phi(x_t)\|_2$ is minimized. Where $v_t$ is a non-negative weight. The algorithm generates samples by the following:
\begin{align}\label{eq:herding_orginal}
    x_{t+1} =\arg\max_{x\in{\cal X}} \langle w_t , \Phi(x) \rangle\notag\\
    w_{t+1}=w_t+\mu-\Phi(x_{t+1}) 
\end{align}

Let ${\cal B}$ denote the convex hull of $X$. \cite{blo12} show that the recursive algorithm in Eq~\eqref{eq:herding_orginal} is equivalent to a Frank-Wolfe algorithm  Problem~\ref{prob:herding_fw}.

\begin{problem}[Herding]\label{prob:herding_fw}
\begin{align*}
     \min_{w\in {\cal B}} \frac{1}{2} \|w-\mu\|_2^2
\end{align*}

We have the following assumptions:
\begin{itemize}
    \item $S=\Phi({\cal X})\subset \R^d$ is a finite feasible set. $|S|=n$.
    \item ${\cal B}={\cal B}(S)\subset \R^{d}$ is the convex hull of the finite set $S\subset \R^d$  defined in Definition~\ref{def:cvx_hull}.
    \item $D_{\max}$ is the maximum diameter of ${\cal B}(S)$ defined in Definition~\ref{def:cvx_hull}
\end{itemize}
\end{problem}

Therefore, a frank-Wolfe algorithm~\cite{blo12} for Herding is proposed as

\begin{algorithm}[H]\caption{Herding Algorithm}\label{alg:herding_formal}
\begin{algorithmic}[1]
\Procedure{Herding}{$S \subset \R^k$}
\State $T\leftarrow O(\frac{ D_{\max}^2}{\epsilon})$, $\forall t \in [T]$
\State $\eta \leftarrow \frac{2}{t+2}$
\State Start with $w^0 \in {\cal B}$. 
\For{ $t=1 \to T-1$}
    \State $s^t \leftarrow \arg\max_{s\in S} \langle w^t-\mu, s \rangle $
    \State $w^{t+1}  \leftarrow (1-\eta)w^{t}+\eta s^t$
\EndFor
\State \Return $w^T$
\EndProcedure
\end{algorithmic}
\end{algorithm}

Algorithm~\ref{alg:herding_formal} takes $O(nd)$ cost per iteration.

To improve the efficiency of Algorithm~\ref{alg:herding_formal}, we propose a Herding algorithm with sublinear cost per iteration using \textit{\lsh}.
\begin{algorithm}[H]\caption{Sublinear Herding Algorithm}\label{alg:herding_lsh_formal}
\begin{algorithmic}[1]
\State {\bf data structure} \textsc{LSH} \Comment{Corollary~\ref{coro:proj_maxip_lsh}}
    \State \hspace{4mm} \textsc{Init}($S \subset \R^d$, $n \in \mathbb{N}$, $d \in \mathbb{N}$, $c\in (0,1)$) \State \Comment{$|S|=n$, $c\in(0,1)$ is {\lsh} parameter, and $d$ is the dimension of data}
    \State \hspace{4mm} \textsc{Query}($x \in \R^d$, $\tau\in(0,1)$) \Comment{$\tau\in(0,1)$ is {\lsh} parameter}
\State {\bf end data structure}
\State
\Procedure{SublinearHerding}{$S \subset \R^d$, $n \in \mathbb{N}$, $d\in \mathbb{N}$,$c\in(0,1)$ ,$\tau\in(0,1)$ } \State\Comment{Theorem~\ref{thm:herding_lsh_formal}}
\State Construct $\phi, \psi : \R^d \rightarrow \R^{d+1}$ as Corollary~\ref{coro:reg_ip_transform}
\State {\bf static} \textsc{LSH} \textsc{lsh}
\State \textsc{lsh}.\textsc{Init}($\psi(S),n,d+3,c$)
\State Start with $w^0 \in {\cal B}$. \Comment{${\cal B}={\cal B}(S) $(see Definition~\ref{def:cvx_hull}).}
\State $T\leftarrow O(\frac{\beta D_{\max}^2}{c^2\epsilon})$, $\forall t \in [T]$
\State $\eta \leftarrow \frac{2}{c(t+2)}$
\For{ $t=1 \to T-1$}
    \State {\color{blue}/* Query with $w^t$ and retrieve its $(c,\phi,\psi)$-$\maxip$ $s^t\in S$ from {\lsh} data structure */}
    \State $s^t \leftarrow \textsc{lsh}.\textsc{Query}(\phi(w^t),\tau)$ 
    \State {\color{blue}/* Update $w^t$ in the chosen direction*/}
    \State $w^{t+1} \leftarrow (1-\eta_t) \cdot w^{t} + \eta_t \cdot s^t$
\EndFor
\State \Return $w^{T}$
\EndProcedure
\end{algorithmic}
\end{algorithm}

\subsection{Convergence Analysis}
The goal of this section is to show the convergence analysis of our Algorithm~\ref{alg:herding_lsh_formal} compare it with Algorithm~\ref{alg:herding_formal} for Herding.

We first show the comparison results in Table~\ref{tab:herd_compare}. In this table, we list the statement, preprocessing time, number of iterations and cost per iteration for our algorithm and original Herding algorithm to converge.
\begin{table}[h]
    \centering
    \begin{tabular}{|l|l|l|l|l|l|} \hline
        {\bf Algorithm}&{\bf Statement} & {\bf Preprocessing} & {\bf \#iters} & {\bf cost per iter}  \\ \hline
        Algorithm~\ref{alg:herding_formal} &\cite{blo12}  & 0 & $O(D_{\max}^2/\epsilon)$ & $O(dn)$  \\ \hline
        Algorithm~\ref{alg:herding_lsh_formal} &Theorem~\ref{thm:herding_lsh_formal} & $O(dn^{1+o(1)}\cdot \kappa)$ & $O(c^{-2}D_{\max}^2/\epsilon)$ & $O(d n^{\rho}\cdot \kappa)$ \\ \hline
    \end{tabular}
    \caption{Comparison between Algorithm~\ref{alg:herding_lsh_formal} and Algorithm~\ref{alg:herding_formal}}
    \label{tab:herd_compare}
\end{table}

Next, we analyze the smoothness of $\frac{1}{2}\|w-\mu\|_2^2$.

\begin{lemma}\label{lemma:1smooth}
We show that $g(w)=\frac{1}{2}\|w^T-\mu\|_2^2$ is a convex  and $1$-smooth function.
\end{lemma}

\begin{proof}
\begin{align}\label{eq:herd_convex_smooth}
g(x)+\langle \nabla g(x),y-x \rangle+\frac{1}{2}\| y-x \|^2_2
=&~\frac{1}{2}\|x-\mu\|_2^2 +\langle x-\mu, y-x \rangle+\frac{1}{2} \|y-x\|_2^2 \notag\\
=&~\frac{1}{2} (x^\top x-2x^\top \mu+\mu^\top\mu)+ (x^\top y-y^\top\mu\notag\\
=&~\frac{1}{2} y^\top y -y^\top\mu +\frac{1}{2} \mu^\top \mu \notag\\
=&~\frac{1}{2} \|y-\mu\|_2^2\notag\\
=&~ g(y)
\end{align}
where all the steps except the last step are reorganizations. The last step follows $g(y)=\frac{1}{2} \|y-\mu\|_2^2$

Rewrite the Eq~\eqref{eq:herd_convex_smooth} above, we have
\begin{align}\label{eq:herd_strong_convex}
    g(y)
    =&~g(x)+\langle \nabla g(x),y-x \rangle+\frac{1}{2}\| y-x \|^2_2\\
    \geq &~g(x)+\langle \nabla g(x),y-x \rangle
\end{align}
$g(x)=\frac{1}{2} \|x-\mu\|_2^2$ is a  convex function.

Rewrite the Eq~\eqref{eq:herd_convex_smooth} above again, we have
\begin{align}\label{eq:herd_smooth}
    g(y)
    =&~g(x)+\langle \nabla g(x),y-x \rangle+\frac{1}{2}\| y-x \|^2_2\\
    \leq &~g(x)+\langle \nabla g(x),y-x \rangle+\frac{1}{2}\| y-x \|^2_2
\end{align}
$g(x)=\frac{1}{2} \|x-\mu\|_2^2$ is a $1$-smooth convex function.
\end{proof}

Next, we show the convergence results of Algorithm~\ref{alg:herding_lsh_formal}.
\begin{theorem}[Convergence result of Sublinear Herding, a formal version of Theorem~\ref{thm:herding_lsh_informal}]\label{thm:herding_lsh_formal}
For any parameters $\epsilon, \delta$, there is an iterative algorithm (Algorithm~\ref{alg:herding_lsh_formal}) for Problem~\ref{prob:herding_fw} that takes $ O( dn^{1+o(1)}\cdot \kappa)$ time in pre-processing and $O((n^{1+o(1)}+dn)\cdot\kappa)$ space, takes $T = O(\frac{D_{\max}^2}{c^2\epsilon})$ iterations and $O(dn^{\rho}\cdot \kappa)$ cost per iteration, starts from a random $w^0$ from ${\cal B}$ as initialization point, updates the $w$ in each iteration
and outputs $w^T \in \R^d$ from ${\cal B}$ such that
\begin{align*}
   \frac{1}{2} \|w^T-\mu\|_2^2 - \min_{w\in \cal{B}} \frac{1}{2} \|w-\mu\|_2^2 \leq \epsilon, 
\end{align*}
holds with probability at least $1-\delta$. Here $\rho:=  \frac{2(1-\tau)^2}{(1-c\tau)^2}-\frac{(1-\tau)^4}{(1-c\tau)^4}+o(1)$ and $\kappa:= \Theta ( \log(T/\delta) )$.
\end{theorem}

\begin{proof}

First,  we show that $g(w)=\frac{1}{2}\|w^T-\mu\|_2^2$ is a convex  and $1$-smooth function. using Lemma~\ref{lemma:1smooth}. Then, we could prove the theorem using Theorem~\ref{thm:herding_lsh_formal}. Following the fact that the computation of gradient is $O(d)$, we could also provide the query time, preprocesisng time and space complexities.
\end{proof}

\subsection{Discussion}
We show that our sublinear Frank-Wolfe algorithm demonstrated in  Algorithm~\ref{alg:herding_lsh_formal} breaks the linear cost per iteration of current Frank-Wolfe algorithm in Algorithm~\ref{alg:herding_formal} in the Herding algorithm. Meanwhile, the extra number of iterations Algorithm~\ref{alg:herding_lsh_formal} pay is affordable.  

Our results show the connection between the extra number of iterations and the cost reduction at each iteration. It represents a formal combination of {\lsh} data structures and Herding algorithm. We hope that this demonstration would provide insights for more applications of {\lsh} in kernel methods and graphical models.

\section{Policy Gradient Optimization}\label{sec:policy}

We present the our results on policy gradient in this section.

\subsection{Problem Formulation}
In this paper, we focus on the action-constrained  Markov Decision Process (ACMDP). In this setting, we are provided with a state ${\cal S}\in \R^{k}$ and action space ${\cal A}\in \R^{d}$, which is the convex hull of $n$-vector. 
However, at each step $t\in \mathbb{N}$, we could only access a finite subset of actions ${\cal C}(s) \subset {\cal A}$ with cardinality $n$. Let us assume the ${\cal C}(s)$ remains the same in each step.
Let us denote $D_{\max}$ as the maximum diameter of ${\cal A}$.

When you play with this ACMDP, the policy you choose is defined as $\pi_{\theta}(s):{\cal S} \rightarrow {\cal A}$ with parameter $\theta$. Meanwhile, there exists a reward function $r: {\cal S}\times {\cal A}  \in [0,1]$. Next, we define the Q function as below,
\begin{align*}
    Q(s,a|\pi_{\theta})&=\E\Big[\sum_{t=0}^{\infty} \gamma^t r(s_t,a_t)| s_0=s,a_0=a,\pi_{\theta}\Big].
\end{align*}
where $\gamma\in (0,1)$ is a discount factor.

Given a state distribution $\mu$, the objective of policy gradient is to maximize the expected value  $J(\mu,\pi_{\theta})=\E_{s\sim \mu,a\sim \pi_{\theta}}[Q(s,a|\pi_{\theta})]$ via policy gradient \cite{slhd+14} denoted as:
\begin{align*}
    \nabla_{\theta} J(\mu,\pi_{\theta})= \E_{s\sim d_{\mu}^{\pi}}\Big[\nabla_{\theta}\pi_{\theta}(s) \nabla_{a}Q(s,\pi_{\theta}(s)|\pi_{\theta})|\Big].
\end{align*}

\cite{lhy21} propose an iterative algorithm that perform $\maxip$ at each iteration $k$ over actions to find
\begin{align}\label{eq:rl_action_maxip_formal}
g_{k}(s)=\max_{a \in {\cal C}(s)} \langle a^k_s-\pi^k_{\theta}(s), \nabla_{a} Q(s,\pi^k_{\theta}(s)|\pi^k_{\theta})) \rangle.
\end{align}

Moreover, \cite{lhy21} also have the following statement
\begin{lemma}[\cite{lhy21}]\label{lem:bound_j_mu}
Given a ACMDP and the gap $g_{k}(s)$ in Eq.\eqref{eq:rl_action_maxip_formal}, we show that
\begin{align*}
    J(\mu,\pi^{k+1}_{\theta})
      \geq &~J(\mu,\pi^k_{\theta}(s))+\frac{(1-\gamma)^2\mu_{\min}^2}{2LD_{\max}^2}\sum_{s\in {\cal S}}g_{k}(s)^2
\end{align*}
\end{lemma}

Therefore, \cite{lhy21} maximize the expected value via minimizing $g_{k}(s)$. 

In this work, we accelerate Eq.~\eqref{eq:rl_action_maxip} using $(c,\phi,\psi,\tau)$-$\maxip$. Here define $\phi: {\cal S}\times \R^d \rightarrow \R^{d+2}$ and $\psi: \R^d \rightarrow \R^{d+3}$ as follows:

\begin{corollary}[Transformation for policy gradient]\label{coro:reg_ip_transform_rl}
Let $g$ be a differential function defined on convex set $\mathcal{K} \subset \R^d$ with maximum diameter $D_\mathcal{K}$. For any $x,y\in \mathcal{K}$, we define  $\phi,\psi: \R^d \rightarrow \R^{d+3}$ as follows:
\begin{align*}
    \phi(x):= \begin{bmatrix}\frac{\phi_0(x)^\top}{D_x}  & 0 & \sqrt{1-\frac{\|\phi_0(x)\|_2^2}{D_x^2}} \end{bmatrix}^\top~~~\psi(y) := \begin{bmatrix} \frac{\psi_0(y)^\top}{D_y} & \sqrt{1-\frac{\|\psi_0(y)\|_2^2}{D_y^2}}  & 0 \end{bmatrix}^\top
\end{align*}
where 
\begin{align*}
\phi_0 (s,\pi^k_{\theta}) :=&~ [ \nabla_{a} Q(s,\pi^k_{\theta}(s)|\pi^k_{\theta})^\top, (\pi^k_{\theta})^\top Q(s,\pi^k_{\theta}(s)|\pi^k_{\theta})]^\top \\
\psi_0(a) =&~[ a^\top,-1]^\top
\end{align*}
and $D_x$ is the maximum diameter of $\phi_0(x)$ and $D_y$ is the maximum diameter of $\psi_0(y)$.

Then, for all $x,y\in \mathcal{K}$ we have $g_{k}(s) =D_xD_y\langle  \phi(s,\pi^k_{\theta}) , \psi(a) \rangle$. Moreover, $\phi(x)$ and $\psi(y)$ are unit vectors with norm $1$.
\end{corollary}

\begin{proof}
We show that
\begin{align*}
    \langle  \phi(s,\pi^k_{\theta}) , \psi(a) \rangle &~= D_x^{-1}D_y^{-1}\langle \nabla_{a} Q(s,\pi^k_{\theta}(s)|\pi^k_{\theta}), a \rangle -\langle \nabla_{a} Q(s,\pi^k_{\theta}(s)|\pi^k_{\theta}), \pi^k_{\theta} \rangle\\
    &~=D_x^{-1}D_y^{-1}\langle a^k_s-\pi^k_{\theta}(s), \nabla_{a} Q(s,\pi^k_{\theta}(s)|\pi^k_{\theta})) \rangle
\end{align*}
where the first step follows the definition of $\phi$ and $\psi$, the second step is an reorganization.
\end{proof}
In this way, we propose a sublinear iteration cost algorithm for policy gradient in Algorithm~\ref{alg:policy_gradient_lsh_formal}.

\begin{algorithm}[H]\caption{Sublinear Frank-Wolfe Policy Optimization (SFWPO)}\label{alg:policy_gradient_lsh_formal}
\begin{algorithmic}[1]
\State {\bf data structure} \textsc{LSH}
\Comment{Corollary~\ref{coro:proj_maxip_lsh}}
    \State \hspace{4mm} \textsc{Init}($S \subset \R^d$, $n \in \mathbb{N}$, $d \in \mathbb{N}$, $c\in (0,1)$) \State \Comment{$|S|=n$, $c\in(0,1)$ is {\lsh} parameter, and $d$ is the dimension of data}
    \State \hspace{4mm} \textsc{Query}($x \in \R^d$, $\tau\in(0,1)$) \Comment{$\tau\in(0,1)$ is {\lsh} parameter}
\State {\bf end data structure}
\State 
\Procedure{SFWPO}{${\cal S} \subset \R^k$, $c\in(0,1)$,$\tau\in (0,1)$}
\State\Comment{Theorem~\ref{thm:policy_gradient_lsh_formal}}
\State {\bf Input:} Initialize the policy parameters as $\theta_0\in R^l$ that satisfies $\pi^0_{\theta}(s)\in {\cal C}(s)$ for all $s\in{\cal S}$
\For{each State $s\in{\cal S}$}
\State Construct $\phi, \psi : \R^d \rightarrow \R^{d+1}$ as Corollary~\ref{coro:reg_ip_transform_rl}
\State {\bf static} \textsc{LSH} $\textsc{lsh}_s$
\State $\textsc{lsh}_s$ \textsc{Init}($\psi({\cal C} (s),n,d+3,c$)
\EndFor

\State $T\leftarrow O(\frac{c^{-2}LD_{\max}^2}{\epsilon^2(1-\gamma)^3\mu_{\min}^2}$
\For{each iteration $k=0,1,\cdots,T$}
\For{each State $s\in{\cal S}$}
  \State Use policy $\pi^k_{\theta}$ and obtain $Q(s,\pi^k_{\theta}(s)|\pi^k_{\theta})$
  \EndFor
  \For{each State $s\in{\cal S}$}
  \State $\hat{a^k_s} \leftarrow \textsc{lsh}_s.\textsc{Query}(\phi(s,\pi^k_{\theta}(s),\tau))$\label{line:lsh_fwpo_query}
  \State $\hat{g_{k}}(s)=\langle \hat{a^k_s}-\pi^k_{\theta}(s), \nabla_{a} Q(s,\pi^k_{\theta}(s)|\pi^k_{\theta})) \rangle$\label{line:g_k_lsh}
  \State $\alpha_k(s)=\frac{(1-\gamma)\mu_{\min}}{L D_s^2}\hat{g_{k}}(s)$
  \State $\pi^{k+1}_{\theta}(s)=\pi^k_{\theta}(s)+\alpha_k(s) (\hat{a^k_s}-\pi^k_{\theta}(s))$\label{line:pi_update_lsh}
  \EndFor
\EndFor
\State \Return $\pi^{T}_{\theta}(s)$
\EndProcedure
\end{algorithmic}
\end{algorithm}

\subsection{Convergence Analysis}
The goal of this section is to show the convergence analysis of of Algorithm~\ref{alg:policy_gradient_lsh_formal} compare it with~\cite{lhy21}.We first show the comparison results in Table~\ref{tab:fwpo_compare}.
\begin{table}[h]
    \centering
    \begin{tabular}{|l|l|l|l|l|l|} \hline
        {\bf Algorithm}&{\bf Statement} & {\bf Preprocessing} & {\bf \#iters} & {\bf cost per iter}  \\ \hline
        \cite{lhy21} &\cite{lhy21} & 0 & $O(\frac{\beta D_{\max}^2}{\epsilon^2(1-\gamma)^3\mu_{\min}^2})$ & $O(dn+{\cal T}_Q)$  \\ \hline
        Algorithm~\ref{alg:policy_gradient_lsh_formal} &Theorem~\ref{thm:policy_gradient_lsh_formal} & $O( dn^{1+o(1)}\cdot \kappa)$ & $O(\frac{c^{-2}\beta D_{\max}^2}{\epsilon^2(1-\gamma)^3\mu_{\min}^2})$ & $O(  dn^{\rho}\cdot \kappa+{\cal T}_Q)$ \\ \hline
    \end{tabular}
    \caption{Comparison between our sublinear policy gradient (Algorithm~\ref{alg:policy_gradient_lsh_formal}) and \cite{lhy21}.}
    \label{tab:fwpo_compare}
\end{table}

The goal of this section is to prove Theorem \ref{thm:policy_gradient_lsh_formal}.

\begin{theorem}[Sublinear Frank-Wolfe Policy Optimization (SFWPO), a formal version of Theorem~\ref{thm:policy_gradient_lsh_informal}]\label{thm:policy_gradient_lsh_formal}
Let ${\cal T}_{Q}$ denote  the time for computing the policy graident. Let $D_{\max}$ denote the maximum diameter of action space and $\beta$ is a constant. Let $\gamma\in (0,1)$. Let $\rho\in(0,1)$ denote  a fixed parameter.  Let $\mu_{\min}$ denote the minimal density of states in ${\cal S}$. There is an iterative algorithm (Algorithm~\ref{alg:policy_gradient_lsh_formal}) that spends $  O(dn^{1+o(1)}\cdot \kappa)$ time in preprocessing and $O((n^{1+o(1)}+dn)\cdot\kappa)$ space, takes $O(\frac{ \beta D_{\max}^2}{\epsilon^2(1-\gamma)^3\mu_{\min}^2})$ iterations and $O( dn^{\rho}\cdot\kappa +{\cal T}_{Q})$ cost per iterations, start from a random point $\pi^0_{\theta}$ as initial point, and output policy $\pi^T_{\theta}$ that is have average gap $\sqrt{\sum_{s\in {\cal S}}g_{T}(s)^2}< \epsilon$ holds 
with probability at least $1-1/\poly(n)$, where $g_{T}(s)$ is defined in Eq.~\eqref{eq:rl_action_maxip_formal}  and $\kappa:= \Theta ( \log(T/\delta) )$.
\end{theorem}

\begin{proof}
Let $\hat{a^k_s}$ denote the action retrieved by {\lsh}. Note that similar to Case 1 of Theorem~\ref{thm:frank_wolfe_lsh_formal}, the algorithms convergences if parameter $\tau$ is greater than maximum inner product. Therefore, we could direct focus on Case 2 and lower bound $\hat{g_{k}}(s)$ as 
\begin{align}\label{eq:g_hat_g_k}
    \hat{g_{k}}(s)=&~\langle \hat{a^k_s}-\pi^k_{\theta}(s), \nabla_{a} Q(s,\pi^k_{\theta}(s)|\pi^k_{\theta})) \rangle\notag\\
    =&~D_xD_y\langle  \phi(s,\pi^k_{\theta}) , \psi(\hat{a^k_s}) \rangle \notag\\
    \geq&~cD_xD_y \max_{a\in {\cal C}(a)} \langle  \phi(s,\pi^k_{\theta}) , \psi(a) \rangle\notag\notag\\
    =&~c \langle a^k_s, \nabla_{a} Q(s,\pi^k_{\theta}(s)|\pi^k_{\theta})) \rangle-c\langle \pi^k_{\theta}(s), \nabla_{a} Q(s,\pi^k_{\theta}(s)|\pi^k_{\theta})) \rangle\notag\\
    =&~ cg_k(s)
\end{align}
where the first step follows from the line~\ref{line:g_k_lsh} in Algorithm~\ref{alg:policy_gradient_lsh_formal}, the second step follows from Corollary~\ref{coro:reg_ip_transform_rl}, the third step follows from Corollary~\ref{coro:proj_maxip_lsh}, the forth step follows from Corollary~\ref{coro:reg_ip_transform_rl}  and the last step is a reorganization.

Next, we upper bound $J(\mu,\pi^{k+1}_{\theta})$ as
\begin{align}\label{eq:bound_j_mu_lsh}
    J(\mu,\pi^{k+1}_{\theta})
      \geq &~J(\mu,\pi^k_{\theta}(s))+\frac{(1-\gamma)^2\mu_{\min}^2}{2LD_{\max}^2}\sum_{s\in {\cal S}}\hat{g_{k}}(s)^2 \notag\\
      \geq&~J(\mu,\pi^k_{\theta}(s))+\frac{c^2(1-\gamma)^2\mu_{\min}^2}{2LD_{\max}^2}\sum_{s\in {\cal S}}g_{k}(s)^2 
\end{align}
where the first step follows from Lemma~\ref{lem:bound_j_mu}, the second step follows from Eq.~\eqref{eq:g_hat_g_k}

Using induction from $1$ to $T$, we have
\begin{align}\label{eq:j_mu_induction_lsh}
    J(\mu,\pi^{T}_{\theta})=J(\mu,\pi^{1}_{\theta})+\frac{c^2(1-\gamma)^2\mu_{\min}^2}{2LD_{\max}^2}\sum_{k=0}^{T}\sum_{s\in {\cal S}}g_{k}(s)^2
\end{align}

Let $G=\sum_{k=0}^{T}\sum_{s\in {\cal S}}g_{k}(s)^2$, we upper bound $G$ as 


\begin{align}\label{eq:bound_G_k^2_lsh}
    G \leq &~\frac{2LD_{\max}^2}{c^2(1-\gamma)^2\mu_{\min}^2}(J(\mu,\pi^{T}_{\theta})-J(\mu,\pi^{0}_{\theta}))\notag\\
    \leq &~\frac{2LD_{\max}^2}{c^2(1-\gamma)^2\mu_{\min}^2}J(\mu,\pi^{*}_{\theta}))\notag\\
    \leq &~\frac{2LD_{\max}^2}{c^2(1-\gamma)^3\mu_{\min}^2}\notag\\
\end{align}
where the first step follows from Eq~\eqref{eq:j_mu_induction_lsh}, the second step follows from $J(\mu,\pi^{*}_{\theta})\geq J(\mu,\pi^{T}_{\theta})$, last step follows from $J(\mu,\pi^{*}_{\theta}) \leq (1-\gamma)^{-1}$.

Therefore, we upper bound $\sum_{s\in {\cal S}}g_{T}(s)^2$ as
\begin{align}
    \sum_{s\in {\cal S}}g_{T}(s)^2
    \leq &~ \frac{1}{T+1}G\notag\\
    \leq &~ \frac{1}{T+1}\frac{2LD_{\max}^2}{c^2(1-\gamma)^3\mu_{\min}^2}
\end{align}
where the first step is a reorganization, the second step follows that $\sum_{s\in {\cal S}}g_{T}(s)^2$ is non-increasing, the second step follows from Eq~\eqref{eq:bound_G_k^2_lsh}. 

If we want $\sum_{s\in {\cal S}}g_{T}(s)^2< \epsilon^2$, $T$ should be $O(\frac{c^{-2}LD_{\max}^2}{\epsilon^2(1-\gamma)^3\mu_{\min}^2})$

 {\bf Preprocessing time}
  According to Corrollary~\ref{coro:proj_maxip_lsh},  can construct $\kappa=\Theta ( \log(T/\delta) )$ {\lsh} data structures for $(c,\phi,\psi,\tau)$-$\maxip$ with $\phi,\psi$ defined in Corollary~\ref{coro:reg_ip_transform_rl}. As transforming every $a\in {\cal A}$ into $\psi(a)$ takes $O(dn)$. Therefore, the total the preprocessing time complexity is $O( d n^{1+o(1)}\cdot \kappa)$.
 
 {\bf Cost per iteration}
  Given each $w^t$, compute the policy gradient takes ${\cal T}_{Q}$. Next, it takes $O(d)$ time to generate $\phi(s,\pi_{\theta}^k)$ according to Corollary~\ref{coro:reg_ip_transform} based on policy gradient. Next, according to Corrollary~\ref{coro:proj_maxip_lsh},  it takes $O(dn^\rho \cdot \kappa)$ to retrieve action from $\kappa=\Theta ( \log(T/\delta) )$ {\lsh} data structures. After we select action, it takes $O(d)$ time to compute the gap the update the value. Thus, the total complexity is $O(dn^{\rho}\cdot \kappa  +{\cal T}_{Q})$ with $\rho:=  \frac{2(1-\tau)^2}{(1-c\tau)^2}-\frac{(1-\tau)^4}{(1-c\tau)^4}+o(1)$.

\end{proof}

\subsection{Discussion}
We show that our sublinear Frank-Wolfe based policy gradient algorithm demonstrated in  Algorithm~\ref{alg:policy_gradient_lsh_formal} breaks the linear cost per iteration of current Frank-Wolfe based policy gradient algorithm algorithm. Meanwhile, the extra number of iterations Algorithm~\ref{alg:policy_gradient_lsh_formal} pay is affordable.

Our results extends the {\lsh} to policy gradient optimization and characterize the relationship between the more iterations we paid and the cost we save at each iteration. These results indicates a formal combination of {\lsh} data structures and policy gradient optimization with theoretical gurantee. We hope that this demonstration would provide new research directions for more applications of {\lsh} in robotics and planning.

\section{More Data Structures: Adaptive \texorpdfstring{$\maxip$}{~} Queries}\label{sec:adaptive}

In optimization, the gradient at each iteration is not independent from the previous gradient. Therefore, it becomes a new setting for using $(c,\tau)$-$\maxip$. If we take the gradient as query and the feasible set as the data set, the queries in each step forms an adaptive sequence. In this way, the failure probability of {\lsh} or other $(c,\tau)$-$\maxip$ data-structures could not be union bounded. To extend $(c,\tau)$-$\maxip$ data-structures such as {\lsh} and graphs into this new setting, we use a query quantization method. This method is standard in various machine learning tasks~\cite{yrkpds21,ssx21}.

We start with relaxing the $(c,\tau)$-$\maxip$ with a inner product error. 

\begin{definition}[Relaxed approximate $\maxip$]\label{def:quantized_approximate_maxip}
Let approximate factor $c\in (0,1)$ and threshold $\tau\in(0,1)$. Let $\lambda\geq 0$ denote an additive error. Given an $n$-vector set $Y \subset \mathbb{S}^{d-1}$, the objective of $(c,\tau,\lambda)$-{$\maxip$} is to construct a data-structure that, for a query $x\in \mathbb{S}^{d-1}$ with conditions that $\max_{y\in Y}\langle x , y \rangle \geq \tau$, it retrieves vector $z \in Y$ that $\langle x , z \rangle \geq c \cdot \max_{y\in Y}\langle x , y \rangle - \lambda$. 
\end{definition}

Then, we present a query quantization approach to solve 
$(c,\tau,\lambda)$-{$\maxip$} for adaptive queries. We assume that the $Q$ is the convex hull of all queries. For any query $x\in Q$, we perform a quantization on it and locate it to the nearest lattice with center $\hat{q}\in Q$. Here the lattice has maximum diameter $2\lambda$. Then, we query $\hat{q}$ on data-structures e.g., {\lsh}, graphs, alias tables. This would generate a $\lambda$ additive error to the inner product. Because the lattice centers are independent, the cumulative failure probability for adaptive query sequence could be union bounded. Formally, we present the corollary as

\begin{corollary}[A query quantization version of Corollary~\ref{coro:maxip_lsh_formal}]\label{coro:maxip_lsh_adaptive} 
Let approximate factor $c\in (0,1)$ and threshold $\tau\in(0,1)$. Given a $n$-vector set $Y \subset \mathbb{S}^{d-1}$, there exits a data-structure with $O(dn^{1+o(1)}\cdot\kappa)$  preprocessing time and $O((n^{1+o(1)}+dn)\cdot\kappa)$ space so that for every query $x$ in an adaptive sequence $X=\{x_1,x_2,\cdots,x_T\}\subset \mathbb{S}^{d-1}$, we take $O(dn^\rho\cdot\kappa)$ query time to solve $(c,\tau,\lambda)$-$\maxip$ with respect to $(x,Y)$ with probability at least $1-\delta$, where $\rho=   \frac{2(1-\tau)^2}{(1-c\tau)^2}-\frac{(1-\tau)^4}{(1-c\tau)^4}+o(1)$, $\kappa:= d\log (nd D_X / (\lambda \delta ) ) $ and $D_X$ is the maximum diameter of all queries in $X$.

\end{corollary}

\begin{proof}
The probability that at least one query $x\in X$ fails is equivalent to the probability that at least one query $\hat{q}\in \hat{Q}$ fails. Therefore, we could union bound the probability as:

\begin{align*}
    \Pr[\exists \hat{q}\in \hat{Q}~~~\textrm{s.t all } ~ (c,\tau)\textsc{-}{\maxip}(q,\hat{Q})~ \mathsf{fail} ]=n \cdot (dD_X/\lambda)^d \cdot (1/10)^{\kappa}\leq \delta
\end{align*}
where the second step follows from $\kappa:= d\log (nd D_X / (\lambda \delta ) ) $.

The results of $\hat{q}$ has a $\lambda$ additive error to the original query. Thus, our results is a $(c,\tau,\lambda)$-{$\maxip$} solution. The time and space complexty is obtained via Corollary~\ref{coro:maxip_lsh_formal}. Thus we finish the proof.
\end{proof}

\begin{definition}[Quantized projected approximate $\maxip$]\label{def:quant_proj_approximate_maxip}
Let approximate factor $c\in (0,1)$ and threshold $\tau\in(0,1)$. Let $\lambda\geq 0$ denote an additive error. Let $\phi, \psi: \R^d \rightarrow \R^k$ denote two transforms. Given an $n$-point dataset $Y \subset \R^d $ so that $\psi(Y) \subset \mathbb{S}^{k-1}$, the goal of the $(c,\phi, \psi,\tau,\lambda)$-{$\maxip$} is to build a data structure that, given a query $x\in \R^d$ and $\phi(x) \in \mathbb{S}^{k-1}$ with the promise that $\max_{y\in Y}\langle \phi(x) , \psi(y) \rangle \geq \tau-\lambda$, it retrieves a vector $z \in Y$ with $\langle \phi(x) , \psi(z) \rangle \geq c \cdot (\phi, \psi)\text{-}\maxip (x,Y)$.
\end{definition}

Next, we extend Corollary~\ref{coro:maxip_lsh_adaptive} to adaptive queries.
\begin{corollary}[]\label{coro:proj_maxip_lsh_adaptive}
Let $c \in (0,1)$, $\tau \in(0,1)$, $\lambda\geq 0$ and $\delta\geq 0$. Let $\phi, \psi: \R^d \rightarrow \R^k$ denote two transforms.   Let ${\cal T}_{\phi}$ denote the time to compute $\phi(x)$ and ${\cal T}_{\psi}$ denote the time to compute $\psi(y)$. Given a set of $n$-points $Y\in \R^d$ with $\psi(Y) \subset {\cal S}^{k-1}$ on the sphere, there exists a data structure with $O(dn^{1+o(1)}\cdot\kappa+{\cal T}_{\psi}n)$ preprocessing time and $O((dn^{1+o(1)} + d n)\cdot\kappa)$ space so that for any query $x \in \R^d$ with $\phi(x)\in {\cal S}^{k-1}$, we take  $O(dn^{\rho}\cdot\kappa+{\cal T}_{\phi})$ query time to solve $(c,\phi,\psi,\tau,\lambda)$-$\maxip$ with respect to $(x,Y)$ with probability at least $1-\delta$, where $\rho:=  \frac{2(1-\tau)^2}{(1-c\tau)^2}-\frac{(1-\tau)^4}{(1-c\tau)^4}+o(1)$, $\kappa:= d\log (nd D_X / (\lambda \delta ) ) $ and $D_X$ is the maximum diameter of all queries in $X$.
\end{corollary}

Finally, we present a modified version of Theorem~\ref{thm:frank_wolfe_lsh_formal}.
\begin{theorem}[Convergence result of Frank-Wolfe via {\lsh} with adaptive input]\label{thm:frank_wolfe_lsh_formal_adaptive}
Let $g : \R^d \rightarrow \R$ denote a convex (see Definition~\ref{def:convex}) and $\beta$-smooth function (see Definition~\ref{def:smooth}). Let the complexity of calculating $\nabla g(x)$ to be ${\cal T}_{g}$. Let  $S \subset \R^d$ denote a set of points with $|S| = n$, and ${\cal B} \subset \R^d$ is the convex hull of $S$ defined in Definition~\ref{def:cvx_hull}. For any parameters $\epsilon, \delta$, there is an iterative algorithm with $(c,\phi,\psi,\tau,c^{-2}\epsilon/4)$-$\maxip$ data structure that takes $O(dn^{1+o(1)}\cdot\kappa)$ preprocessing time and $O((n^{1+o(1)}+dn)\cdot\kappa)$ space, takes $T = O(\frac{\beta D_{\max}^2}{\epsilon})$ iterations and $O (d n^{\rho}\cdot\kappa+{\cal T}_{g})$ cost per iteration, starts from a random $w^0$ from ${\cal B}$ as initialization point, updates the $w$ in each iteration as follows:
\begin{align*}
    s^{t} \leftarrow & ~  (c,\phi,\psi,\tau,c^{-2}\epsilon/4)\text{-}\maxip\ \text{of } w^t\ \text{with respect to}\ S \\
    w^{t+1} \leftarrow & ~ w^t +  \eta \cdot (s^t-w^t)
\end{align*}
and outputs $w^T \in \R^d$ from ${\cal B}$ such that
\begin{align*}
   g(w^T) - \min_{w\in \cal{B}} g(w) \leq \epsilon, 
\end{align*}
holds with probability at least $1-\delta$. Here $\kappa:= d\log (nd D_X / (\lambda \delta ) ) $ and $\rho:=  \frac{2(1-\tau)^2}{(1-c\tau)^2}-\frac{(1-\tau)^4}{(1-c\tau)^4}+o(1)$.
\end{theorem}

\begin{proof}
{\bf Convergence.}

We start with modifying Eq.~\eqref{eq:bound_h_t} with additive $\maxip$ error $\lambda$ and get
\begin{align*}
     h_{t+1}&=(1-c\eta_t)h_{t}+\frac{\beta D_{\max}^2}{2}\eta_t^2+\eta_t\lambda
\end{align*}

Let $e_t=A_th_t$ with $A_t=\frac{t(t+1)}{2}$. Let $\eta_t=\frac{2}{c(t+2)}$. Let $\lambda=\frac{\beta D_{\max}^2}{T+1}$ Following the proof in Theorem~\ref{thm:frank_wolfe_lsh_formal}, we upper bound $e_{t+1}-e_t$ as 
\begin{align}
    e_{t+1}-e_t\leq&~\left(A_{t+1}(1-c\eta_t)-A_{t}\right)h_{t}+ \frac{\beta D_{\max}^2}{2}A_{t+1}\eta_t^2+A_{t+1}\eta_t\lambda
\end{align}
where
\begin{itemize}
    \item $A_{t+1}(1-\eta_t)-A_{t}=0$
    \item $A_{t+1}\frac{\eta_t^2}{2}=\frac{t+1}{(t+2)c^2}<c^{-2}$
    \item $A_{t+1}\eta_t\lambda= (t+1)\lambda<\beta D_{\max}^2$.
\end{itemize}

Therefore, 
\begin{align}\label{eq:original_fw_et_final_lsh_ada}
    e_{t+1}-e_t<&~2c^{-2}\beta D_{\max}^2
\end{align}

Based on Eq~\eqref{eq:original_fw_et_final_lsh_ada}, we upper bound $e_t$ using induction and have
\begin{align}
    e_t<2c^{-2}t\beta D_{\max}^2
\end{align}
Using the definition of $e_t$, we have
\begin{align}
     h_t=\frac{e_t}{A_t}<\frac{4\beta D_{\max}^2}{c^{2}(t+1)}
\end{align}

To make $h_T\leq \epsilon$, $T$ should be in $O(\frac{\beta D_{\max}^2
}{c^2\epsilon})$. Moreover, $\lambda=\frac{\beta D_{\max}^2}{T+1}=\frac{\epsilon}{4c^2}$. 

{\bf Preprocessing time}
  According to Corrollary~\ref{coro:proj_maxip_lsh_adaptive},  can construct $\kappa= d\log (nd D_X / (\lambda \delta )$ {\lsh} data structures for $(c,\phi,\psi,\tau,c^{-2}\epsilon/4)$-$\maxip$ with $\phi,\psi$ defined in Corollary~\ref{coro:reg_ip_transform}. As transforming every $s\in S$ into $\psi(s)$ takes $O(dn)$. Therefore, the total the preprocessing time complexity is $O( d n^{1+o(1)}\cdot \kappa)$ and space complexity is $O((n^{1+o(1)}+dn)\cdot\kappa)$.
 
 {\bf Cost per iteration}
  Given each $w^t$, compute $\nabla g(w^t)$ takes ${\cal T}_{g}$. Next, it takes $O(d)$ time to generate $\phi(w^t)$ according to Corollary~\ref{coro:reg_ip_transform} based on $g(w^t)$ and $\nabla g(w^t)$. Next, according to Corrollary~\ref{coro:proj_maxip_lsh_adaptive},  it takes $O(dn^\rho \cdot \kappa)$ to retrieve $s^t$ from $\kappa$ {\lsh} data structures. After we select $s^t$, it takes $O(d)$ time to update the $w^{t+1}$. Combining the time for gradient calculation, {\lsh} query and $w^t$ update, the total complexity is $O(dn^{\rho}\cdot \kappa  +{\cal T}_{g})$ with $\rho:=  \frac{2(1-\tau)^2}{(1-c\tau)^2}-\frac{(1-\tau)^4}{(1-c\tau)^4}+o(1)$.
 \end{proof}

Similarly, we could extend the results to statements of Herding algorithm and policy gradient.

\begin{theorem}[Modified result of Sublinear Herding]\label{thm:herding_lsh_adaptive}
For any parameters $\epsilon, \delta$, there is an iterative algorithm (Algorithm~\ref{alg:herding_lsh_formal}) for Problem~\ref{prob:herding_fw}  with $c^{-2}\epsilon/4$ query quantization that takes $ O( dn^{1+o(1)}\cdot \kappa)$ time in pre-processing and $O((n^{1+o(1)}+dn)\cdot\kappa)$ space, takes $T = O(\frac{D_{\max}^2}{c^2\epsilon})$ iterations and $O(dn^{\rho}\cdot \kappa)$ cost per iteration, starts from a random $w^0$ from ${\cal B}$ as initialization point, updates the $w$ in each iteration based on Algorithm~\ref{alg:herding_lsh_formal}
and outputs $w^T \in \R^d$ from ${\cal B}$ such that
\begin{align*}
   \frac{1}{2} \|w^T-\mu\|_2^2 - \min_{w\in \cal{B}} \frac{1}{2} \|w-\mu\|_2^2 \leq \epsilon, 
\end{align*}
holds with probability at least $1-\delta$. Here $\rho:=  \frac{2(1-\tau)^2}{(1-c\tau)^2}-\frac{(1-\tau)^4}{(1-c\tau)^4}+o(1)$ and $\kappa:= d\log (nd D_X / (\lambda \delta ) )$.
\end{theorem}

\begin{theorem}[Modified result of Sublinear Frank-Wolfe Policy Optimization (SFWPO)]\label{thm:policy_gradient_lsh_adaptive}
Let ${\cal T}_{Q}$ denote the time for computing the policy graident. Let $D_{\max}$ denote the maximum diameter of action space and $\beta$ is a constant. Let $\gamma\in (0,1)$. Let $\rho\in(0,1)$ denote a fixed parameter.  Let $\mu_{\min}$ denote the minimal density of sates in ${\cal S}$. There is an iterative algorithm (Algorithm~\ref{alg:policy_gradient_lsh_formal}) with $c^{-2}\epsilon/4$ query quantization that spends $  O(dn^{1+o(1)}\cdot \kappa)$ time in preprocessing and $O((n^{1+o(1)}+dn)\cdot\kappa)$ space, takes $O(\frac{ \beta D_{\max}^2}{\epsilon^2(1-\gamma)^3\mu_{\min}^2})$ iterations and $O( dn^{\rho}\cdot \kappa +{\cal T}_{Q})$ cost per iterations, start from a random point $\pi^0_{\theta}$ as initial point, and output policy $\pi^T_{\theta}$ that has average gap $\sqrt{\sum_{s\in {\cal S}}g_{T}(s)^2}< \epsilon$ holds 
with probability at least $1-1/\poly(n)$, where $g_{T}(s)$ is defined in Eq.~\eqref{eq:rl_action_maxip_formal} and $\kappa:= d\log (nd D_X / (\lambda \delta ) )$.
\end{theorem}

\end{document}